\documentclass[sigconf,authorversion,nonacm]{aamas}

\usepackage{balance} 

\setcopyright{ifaamas}
\acmConference[AAMAS '24]{Proc.\@ of the 23rd International Conference
on Autonomous Agents and Multiagent Systems (AAMAS 2024)}{May 6 -- 10, 2024}
{Auckland, New Zealand}{N.~Alechina, V.~Dignum, M.~Dastani, J.S.~Sichman (eds.)}
\copyrightyear{2024}
\acmYear{2024}
\acmDOI{}
\acmPrice{}
\acmISBN{}

\acmSubmissionID{252}

\usepackage{subcaption}
\usepackage[inline]{enumitem}
\usepackage{algorithm}
\usepackage[noend]{algpseudocode}

\newtheorem{theorem}{Theorem}

\newtheorem{proposition}{Proposition}

\DeclareMathOperator*{\argmax}{arg\,max}

\newcommand{\bcc}[1]{\left\{{#1}\right\}}
\newcommand{\brr}[1]{\left({#1}\right)}
\newcommand{\bss}[1]{\left[{#1}\right]}
\newcommand{\ipp}[2]{\left\langle{#1},{#2}\right\rangle}
\newcommand\norm[1]{\left\lVert#1\right\rVert}
\newcommand{\abs}[1]{\left\vert#1\right\vert}

\newcommand{\Expect}[1]{\mathbb{E}\bss{{#1}}}
\newcommand{\Expectover}[2]{\mathbb{E}_{#1}\!\left[#2\right]}

\newcommand{\envfourroom}{\textsc{Room}}

\newcommand{\envlinekey}{\textsc{LineK}}

\newcommand{\orig}{\textsc{Orig}}
\newcommand{\EXPRD}{\textsc{ExpRD}}
\newcommand{\Invariance}{\textsc{Invar}}
\newcommand{\AlgOurs}{\textsc{ExpAdaRD}}

\allowdisplaybreaks

\title[Informativeness of Reward Functions in Reinforcement Learning]{Informativeness of Reward Functions in Reinforcement Learning}

\author{Rati Devidze}
\affiliation{
  \institution{MPI-SWS}
  \city{Saarbr{\"u}cken}
  \country{Germany}}
\email{rdevidze@mpi-sws.org}

\author{Parameswaran Kamalaruban}
\affiliation{
  \institution{The Alan Turing Institute}
  \city{London}
  \country{United Kingdom}}
\email{pkamalaruban@gmail.com}

\author{Adish Singla}
\affiliation{
  \institution{MPI-SWS}
  \city{Saarbr{\"u}cken}
  \country{Germany}}
\email{adishs@mpi-sws.org}

\begin{abstract}
Reward functions are central in specifying the task we want a reinforcement learning agent to perform. Given a task and desired optimal behavior, we study the problem of designing informative reward functions so that the designed rewards speed up the agent's convergence. In particular, we consider expert-driven reward design settings where an expert or teacher seeks to provide informative and interpretable rewards to a learning agent. Existing works have considered several different reward design formulations; however, the key challenge is formulating a reward informativeness criterion that adapts w.r.t. the agent's current policy and can be optimized under specified structural constraints to obtain interpretable rewards. In this paper, we propose a novel reward informativeness criterion, a quantitative measure that captures how the agent's current policy will improve if it receives rewards from a specific reward function. We theoretically showcase the utility of the proposed informativeness criterion for adaptively designing rewards for an agent. Experimental results on two navigation tasks demonstrate the effectiveness of our adaptive reward informativeness criterion.
\end{abstract}


\newcommand{\BibTeX}{\rm B\kern-.05em{\sc i\kern-.025em b}\kern-.08em\TeX}

\makeatletter
\gdef\@copyrightpermission{
	\begin{minipage}{0.3\columnwidth}
		\href{https://creativecommons.org/licenses/by/4.0/}{\includegraphics[width=0.90\textwidth]{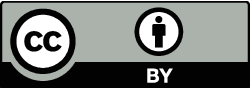}}
	\end{minipage}\hfill
	\begin{minipage}{0.7\columnwidth}
		\href{https://creativecommons.org/licenses/by/4.0/}{This work is licensed under a Creative Commons Attribution International 4.0 License.}
	\end{minipage}
	\vspace{5pt}
}
\makeatother

\begin{document}


\pagestyle{fancy}
\fancyhead{}


\maketitle 

\section{Introduction} \label{sec:intro}

Reward functions play a central role during the learning/training process of a reinforcement learning (RL) agent. Given a task the agent is expected to perform, many different reward functions exist under which an optimal policy has the same performance on the task. This freedom in choosing a reward function for the task, in turn, leads to the fundamental question of designing appropriate rewards for the RL agent that match certain desired criteria~\cite{DBLP:conf/icml/Mataric94,DBLP:conf/icml/RandlovA98,DBLP:conf/icml/NgHR99}. In this paper, we study the problem of designing \emph{informative} reward functions so that the designed rewards speed up the agent's convergence~\cite{DBLP:conf/icml/Mataric94,DBLP:conf/icml/RandlovA98,DBLP:conf/icml/NgHR99,DBLP:conf/icml/LaudD03,DBLP:conf/nips/DaiW19,DBLP:conf/nips/Arjona-MedinaGW19}.

More concretely, we focus on expert-driven reward design settings where an expert or teacher seeks to provide informative rewards to a learning agent~\cite{DBLP:conf/icml/Mataric94,DBLP:conf/icml/NgHR99,DBLP:conf/aaai/ZhangP08,DBLP:conf/sigecom/ZhangPC09,DBLP:conf/ijcai/GoyalNM19,DBLP:conf/nips/MaZSZ19,DBLP:conf/icml/RakhshaRD0S20,DBLP:journals/corr/abs-2011-10824,jiang2021temporal_AAAI,devidze2021explicable}. In  expert-driven reward design settings, the designed reward functions should also satisfy certain \emph{structural constraints} apart from being informative, e.g., to ensure interpretability of reward signals or to match required reward specifications~\cite{grzes2008plan,DBLP:conf/atal/0002CP19,camacho2017decision,DBLP:conf/nips/JothimuruganAB19,jiang2021temporal_AAAI,devidze2021explicable,DBLP:journals/jair/IcarteKVM22,BewleyL22}. For instance, informativeness and interpretability become crucial in settings where rewards are designed for human learners who are learning to perform sequential tasks in pedagogical applications such as educational games~\cite{DBLP:conf/chi/ORourkeHBDP14} and open-ended problem solving domains~\cite{maloney2008programming}. Analogously, informativeness and structural constraints become crucial in settings where rewards are designed for complex compositional tasks in the robotics domain that involve reward specifications in terms of automata or subgoals~\cite{jiang2021temporal_AAAI,DBLP:journals/jair/IcarteKVM22}. To this end, an important research question is: \emph{How to formulate reward informativeness criterion
that can be optimized under specified structural constraints?}

\looseness-1Existing works have considered different reward design formulations; however, they have limitations in appropriately incorporating informativeness and structural properties. On the one hand, potential-based reward shaping (PBRS) is a well-studied family of reward design techniques~\cite{DBLP:conf/icml/NgHR99,DBLP:journals/jair/Wiewiora03,DBLP:conf/aaai/AsmuthLZ08,grzes2008plan,DBLP:conf/aamas/DevlinK12,DBLP:conf/atal/Grzes17,DBLP:conf/atal/0002CP19,DBLP:conf/ijcai/GoyalNM19,jiang2021temporal_AAAI}. While PBRS techniques enable designing informative rewards via utilizing informative potential functions (e.g., near-optimal value function for the task), the resulting reward functions do not adhere to specific structural constraints. On the other hand, optimization-based reward design techniques is another popular family of techniques~\cite{DBLP:conf/aaai/ZhangP08,DBLP:conf/sigecom/ZhangPC09,DBLP:conf/nips/MaZSZ19,DBLP:conf/icml/RakhshaRD0S20,DBLP:journals/corr/abs-2011-10824,devidze2021explicable}. While optimization-based techniques enable enforcing specific structural constraints, there is a lack of suitable reward informativeness criterion that is amenable to optimization as part of these techniques. In this family of techniques, a recent work \cite{devidze2021explicable} introduced a reward informativeness criterion suitable for optimization under sparseness structure; however, their informativeness criterion doesn't account for the agent's current policy, making the reward design process agnostic to the agent's learning progress.

In this paper, we present a general framework, \AlgOurs{}, for \emph{expert-driven explicable and adaptive reward design}. \AlgOurs{} utilizes a novel reward informativeness criterion, a quantitative measure that captures how the agent's current policy will improve if it receives rewards from a specific reward function. Crucially, the informativeness criterion adapts w.r.t. the agent’s current policy and can be optimized under specified structural constraints to obtain interpretable rewards.  Our main results and contributions are:
\begin{enumerate}[label={\Roman*.},leftmargin=13pt]
\item We introduce a reward informativeness criterion formulated within bi-level optimization. By analyzing it for a specific learning algorithm, we derive a novel informativeness criterion that is amenable to the reward optimization process (Sections~\ref{subsec:info-criterion-proposal}~and~\ref{subsec:info-criterion-intuitive}). 
\item \looseness-1We theoretically showcase the utility of our informativeness criterion in adaptively designing rewards by analyzing the convergence speed up of an agent in a simplified setting (Section~\ref{subsec:info-criterion-expert}).
\item We empirically demonstrate the effectiveness of our reward informativeness criterion for designing explicable and adaptive reward functions in two navigation environments. (Section~\ref{sec:evaluation}).\footnote{Github: \url{https://github.com/machine-teaching-group/aamas2024-informativeness-of-reward-functions}.\label{footnote:gitcode}}
\end{enumerate}

\subsection{Related Work}
\label{sec:related_work}

\looseness-1\paragraph{\textbf{Expert-driven reward design.}} As previously discussed, well-studied families of expert-driven reward design techniques include potential-based reward shaping (PBRS)~\cite{DBLP:conf/icml/NgHR99,DBLP:journals/jair/Wiewiora03,DBLP:conf/aaai/AsmuthLZ08,grzes2008plan,DBLP:conf/aamas/DevlinK12,DBLP:conf/atal/Grzes17,DBLP:conf/atal/0002CP19,DBLP:conf/ijcai/GoyalNM19,jiang2021temporal_AAAI}, optimization-based techniques~\cite{DBLP:conf/aaai/ZhangP08,DBLP:conf/sigecom/ZhangPC09,DBLP:conf/nips/MaZSZ19,DBLP:conf/icml/RakhshaRD0S20,DBLP:journals/corr/abs-2011-10824,devidze2021explicable}, and reward shaping with expert demonstrations or feedback~\cite{brys2015reinforcement,de2020imitation,daniel2014active,xiao2020fresh}. Our reward design framework, \AlgOurs{}, also uses an optimization-based design process. The key issue with existing optimization-based techniques is a lack of suitable reward informativeness criterion. A recent work~\cite{devidze2021explicable} introduced an expert-driven explicable reward design framework (\EXPRD{}) that optimizes an informativeness criterion under sparseness structure. However, their informativeness criterion doesn't account for the agent's current policy, making the reward design process agnostic to the agent's learning progress. In contrast, we propose an adaptive informativeness criterion enabling it to provide more informative reward signals. Technically, our proposed reward informativeness criterion is quite different from that proposed in ~\cite{devidze2021explicable} and is derived based on analyzing meta-gradients within bi-level optimization formulation. 

\looseness-1\paragraph{\textbf{Learner-driven reward design.}} Learner-driven reward design techniques involve an agent designing its own rewards throughout the training process to accelerate convergence~\cite{DBLP:books/sp/13/Barto13,DBLP:conf/nips/KulkarniNST16,DBLP:conf/nips/TrottZXS19,DBLP:conf/nips/Arjona-MedinaGW19,DBLP:conf/ijcai/FerretMGP20,sorg2010online_gradient_asc,zheng2018learning,memarian2021self,devidze2022exploration}. These learner-driven techniques employ various strategies, including designing intrinsic rewards based on exploration bonuses~\cite{DBLP:books/sp/13/Barto13,DBLP:conf/nips/KulkarniNST16,DBLP:conf/nips/ZhangMS20}, crafting rewards using domain-specific knowledge~\cite{DBLP:conf/nips/TrottZXS19}, using credit assignment to create intermediate rewards~\cite{DBLP:conf/nips/Arjona-MedinaGW19,DBLP:conf/ijcai/FerretMGP20}, and designing parametric reward functions by iteratively updating reward parameters and optimizing the agent's policy based on learned rewards~\cite{sorg2010online_gradient_asc,zheng2018learning,memarian2021self,devidze2022exploration}. While these learner-driven techniques are typically designing adaptive and online reward functions, these techniques do not emphasize the formulation of an informativeness criterion explicitly. In our work, we draw on insights from meta-gradient derivations presented in ~\cite{sorg2010online_gradient_asc,zheng2018learning,memarian2021self,devidze2022exploration} to develop an adaptive informativeness criterion tailored for the expert-driven reward design settings.
\section{Preliminaries}
\label{sec:formal-setup}
\paragraph{\textbf{Environment.}} An environment is defined as a Markov Decision Process (MDP) denoted by $M := \brr{\mathcal{S},\mathcal{A},T,P_0,\gamma,R}$, where $\mathcal{S}$ and $\mathcal{A}$ represent the state and action spaces respectively. The state transition dynamics are captured by $T: \mathcal{S} \times \mathcal{S} \times \mathcal{A} \rightarrow \bss{0,1}$, where $T(s' \mid s,a)$ denotes the probability of transitioning to state $s'$ by taking action $a$ from state $s$. The discounting factor is denoted by $\gamma$, and $P_0$ represents the initial state distribution. The reward function is given by $R: \mathcal{S} \times \mathcal{A} \rightarrow \mathbb{R}$.

\paragraph{\textbf{Policy and performance.}} We denote a stochastic policy $\pi: \mathcal{S} \rightarrow \Delta \brr{\mathcal{A}}$ as a mapping from a state to a probability distribution over actions, and a deterministic policy $\pi: \mathcal{S} \rightarrow \mathcal{A}$ as a mapping from a state to an action. For any trajectory $\xi = \bcc{(s_t, a_t)}_{t = 0,1,\dots,H}$, we define its cumulative return with respect to reward function $R$ as $J(\xi, R) := \sum_{t=0}^H \gamma^t \cdot R(s_t, a_t)$. The expected cumulative return (value) of a policy $\pi$ with respect to $R$ is then defined as $J(\pi, R) := \Expect{J(\xi, R) | P_0, T, \pi}$, where $s_0\!\sim\!P_0(\cdot)$, $a_t\!\sim\!\pi (\cdot | s_t)$, and $s_{t+1}\!\sim\!T(\cdot | s_t, a_t)$. A learning agent (learner) in our setting seeks to find a policy that has maximum value with respect to $R$, i.e., $\max_\pi J(\pi, R)$. We denote the state occupancy measure of a policy $\pi$ by $d^\pi$. Furthermore, we define the state value function $V^\pi_R$ and the action value function $Q^\pi_R$ of a policy $\pi$ with respect to $R$ as follows, respectively: $V^\pi_R (s) = \mathbb{E}[J(\xi, R) | s_0 =s, T, \pi]$ and $Q^\pi_R (s,a) = \mathbb{E}[J(\xi, R) | s_0 =s, a_0 = a, T, \pi]$. The optimal value functions are given by $V^*_R (s) = \sup_\pi V^\pi_R (s)$ and $Q^*_R (s,a) = \sup_\pi Q^\pi_R (s,a)$.  

\section{Expert-driven Explicable and Adaptive Reward Design}
\label{sec:expadard-formal}

In this section, we present a general framework for expert-driven reward design, \AlgOurs{}, as outlined in Algorithm~\ref{alg:expert-driven-adaptive-reward-design}. In our framework, an expert or teacher seeks to provide informative and interpretable rewards to a learning agent. In each round $k$, we address a reward design problem involving the following key elements: an underlying reward function $\overline{R}$, a target policy $\pi^T$ (e.g., a near-optimal policy w.r.t. $\overline{R}$), a learner's policy $\pi^L_{k-1}$, and a learning algorithm $L$. The main objective of this reward design problem is to craft a new reward function $R_k$ under constraints $\mathcal{R}$ such that $R_k$ provides informative learning signals when employed to update the policy $\pi^L_{k-1}$ using the algorithm $L$. To quantify this objective, it is essential to define a reward informativeness criterion, $I_L(R \mid \overline{R}, \pi^T, \pi^L_{k-1})$, that adapts w.r.t. the agent's current policy and can be optimized under specified structural constraints to obtain interpretable rewards.
Given this informativeness criterion $I_L$ (to be developed in Section~\ref{sec:info-criterion}), the reward design problem can be formulated as follows:
\begin{equation}
\max_{R \in \mathcal{R}} I_L(R \mid \overline{R}, \pi^T, \pi^L_{k-1}) .
\label{eq:reward-design-problem-orig}
\end{equation}
\looseness-1Here, the set $\mathcal{R}$ encompasses additional constraints tailored to the application-specific requirements, including (i) policy invariance constraints $\mathcal{R}_\textnormal{inv}$ to guarantee that the designed reward function induces the desired target policy and (ii) structural constraints $\mathcal{R}_\textnormal{str}$ to obtain interpretable rewards, as further discussed below.

\begin{algorithm}[t!]
    \caption{A General Framework for Expert-driven Explicable and Adaptive Reward Design (\AlgOurs{})}
    \begin{algorithmic}[1]
        \State \textbf{Input:} \looseness-1MDP $M := \big( \mathcal{S},\mathcal{A},T,P_0,\gamma,\overline{R} \big)$, target policy $\pi^T$, learning algorithm $L$, informativeness criterion $I_L$, reward constraint set $\mathcal{R}$
        \State \textbf{Initialize:} learner's initial policy $\pi_{0}^L$ 
        \For{$k = 1,2,\dots, K$}
            \State Expert/teacher updates the reward function: \qquad \qquad \qquad $R_k \gets \argmax_{R \in \mathcal{R}} I_L(R \mid \overline{R}, \pi^T, \pi^L_{k-1})$
            \State Learner updates the policy: $\pi^L_k \gets L(\pi^L_{k-1}, R_k)$ 
        \EndFor{}
        \State \textbf{Output:} learner's policy $\pi^L_K$
    \end{algorithmic}
    \label{alg:expert-driven-adaptive-reward-design}
\end{algorithm}

\paragraph{\textbf{Invariance constraints.}} Let $\overline{\Pi}^* := \{\pi: \mathcal{S} \rightarrow \mathcal{A} \text{ s.t. } V^{\pi}_{\overline{R}} (s) = V^*_{\overline{R}} (s), \forall s \in \mathcal{S}\}$ denote the set of all deterministic optimal policies under $\overline{R}$. Next, we define $\mathcal{R}_\textnormal{inv}$ as a set of invariant reward functions, where each $R \in \mathcal{R}_\textnormal{inv}$ satisfies the following conditions~\cite{DBLP:conf/icml/NgHR99,devidze2021explicable}: 
\begin{align*}
Q^{\pi^T}_{R}(s, a) - V^{\pi^T}_{R}(s) ~\leq~& Q^{\pi^T}_{\overline{R}}(s, a) - V^{\pi^T}_{\overline{R}}(s), \quad \forall a \in \mathcal{A}, s \in \mathcal{S}.
\end{align*}
When $\pi^T$ is an optimal policy under $\overline{R}$ (i.e., $\pi^T \in \overline{\Pi}^*$), these conditions guarantee the following: (i) $\pi^T$ is an optimal policy under $R$; (ii) any optimal policy induced by $R$ is also an optimal policy under $\overline{R}$; (iii) reward function $\overline{R} \in \mathcal{R}_\textnormal{inv}$, i.e., $\mathcal{R}_\textnormal{inv}$ is non-empty.\footnote{\looseness-1We can guarantee point (i) of $\pi^T$ being an optimal policy under $R$ by replacing the right-hand side with $-\epsilon$ for $\epsilon > 0$; however, this would not guarantee (ii) and (iii).}

\paragraph{\textbf{Structural constraints.}} We consider structural constraints as a way to obtain interpretable rewards (e.g., sparsity or tree-structured rewards) and satisfy application-specific requirements (e.g., bounded rewards). We denote the set of reward functions conforming to specified structural constraints as $\mathcal{R}_\textnormal{str}$~\cite{grzes2008plan,camacho2017decision,DBLP:conf/atal/0002CP19,DBLP:conf/nips/JothimuruganAB19,DBLP:journals/jair/IcarteKVM22,jiang2021temporal_AAAI,devidze2021explicable,BewleyL22}. We implement these constraints via a set of parameterized reward functions, denoted as $\mathcal{R}_\textnormal{str} = \{R_\phi: \mathcal{S} \times \mathcal{A} \to \mathbb{R} \text{ where } \phi \in \mathbb{R}^d\}$. For example, given a feature representation $f: \mathcal{S} \times \mathcal{A} \to \{0,1\}^d$, we employ $R_\phi(s,a) = \ipp{\phi}{f(s,a)}$ in our experimental evaluation (Section~\ref{sec:evaluation}). In particular, we will use different feature representations to specify constraints induced by coarse-grained state abstraction~\cite{DBLP:journals/corr/abs-2006-13160} and tree structure~\cite{BewleyL22}. Furthermore, it is possible to impose additional constraints on $\phi$, such as bounding its $\ell_\infty$ norm by $R_\mathrm{max}$ or requiring that its support $\mathrm{supp}(\phi)$, defined as $\{i: i \in [d], \phi_{i} \neq 0\}$, matches a predefined set $\mathcal{Z} \subseteq [d]$~\cite{,devidze2021explicable}. 

\section{Informativeness Criterion for Reward Design}
\label{sec:info-criterion}

In this section, we focus on developing a reward informativeness criterion that can be optimized for the reward design formulation in Eq.~\eqref{eq:reward-design-problem-orig}. We first introduce an informativeness criterion formulated within a bi-level optimization framework and then propose an intuitive informativeness criterion that can be generally applied to various learning algorithms. 

\paragraph{\textbf{Notation.}} In the subscript of the expectations $\mathbb{E}$, let $\pi(a|s)$ mean $a \sim \pi(\cdot | s)$, $\mu^\pi(s, a)$ mean $s \sim d^\pi, a \sim \pi(\cdot | s)$, and $\mu^\pi(s)$ mean $s \sim d^\pi$. Further, we use shorthand notation $\mu^\pi_{s,a}$ and $\mu^\pi_s$ to refer $\mu^{\pi}(s,a)$ and $\mu^{\pi}(s)$, respectively.

\subsection{Bi-Level Formulation for Reward Informativeness $I_L(R)$}
\label{subsec:info-criterion-proposal}

We consider parametric reward functions of the form $R_\phi: \mathcal{S} \times \mathcal{A} \to \mathbb{R}$, where $\phi \in \mathbb{R}^d$, and parametric policies of the form $\pi_\theta: \mathcal{S} \rightarrow \Delta \brr{\mathcal{A}}$, where $\theta \in \mathbb{R}^n$. Let $\overline{R}$ be the underlying reward function, and let $\pi^T$ be a target policy (e.g., a near-optimal policy w.r.t. $\overline{R}$). We measure the performance of any policy $\pi_\theta$ w.r.t. $\overline{R}$ and $\pi^T$ using the following performance metric: $J(\pi_\theta; \overline{R}, \pi^T) = \Expectover{\mu^{\pi^T}_s}{\Expectover{\pi_\theta(a|s)}{A^{\pi^T}_{\overline{R}}(s,a)}}$, where ${A}^{\pi^T}_{\overline{R}} (s,a) = {Q}^{\pi^T}_{\overline{R}} (s,a) - {V}^{\pi^T}_{\overline{R}} (s)$ is the advantage function of policy $\pi^T$ w.r.t. $\overline{R}$. Given a current policy $\pi_\theta$ and a reward function $R$, the learner updates the policy parameter using a learning algorithm $L$ as follows: $\theta_\textnormal{new} \gets L(\theta, R)$.

To evaluate the informativeness of a reward function $R_\phi$ in guiding the convergence of the learner's policy $\pi^L := \pi_{\theta^L}$ towards the target policy $\pi^T$, we define the following informativeness criterion:
\begin{align}
I_L(R_\phi \mid \overline{R}, \pi^T, \pi^L) ~:=~& J(\pi_{\theta^L_\textnormal{new}(\phi)}; \overline{R}, \pi^T) \nonumber\\
&\text{where} \quad \theta^L_\textnormal{new}(\phi) \gets L(\theta^L, R_\phi) .
\label{eq:intuitive-IR-bi-level}
\end{align}
\looseness-1The above criterion measures the performance of the resulting policy after the learner updates $\pi^L$ using the reward function $R_\phi$. However, this criterion relies on having access to the learning algorithm $L$ and evaluating this criterion requires potentially expensive policy updates using $L$. In the subsequent analysis, we further examine this criterion to develop an intuitive alternative that is independent of any specific learning algorithm and does not require any policy updates for its evaluation.  


\looseness-1\paragraph{\textbf{Analysis for a specific learning algorithm $L$.}} Here, we present an analysis of the informativeness criterion defined above, considering a simple learning algorithm $L$. Specifically, we consider an algorithm $L$ that utilizes parametric policies $\bcc{\pi_\theta: \theta \in \mathbb{R}^n}$ and performs single-step vanilla policy gradient updates using $Q$-values computed using $h$-depth planning~\cite{sorg2010online_gradient_asc,zheng2018learning,devidze2022exploration}. We update the policy parameter $\theta$ by employing a reward function $R$ in the following manner:
\begin{align*}
L(\theta, R) ~:&=~ \theta + \alpha \cdot \big[ \nabla_\theta J(\pi_\theta, R) \big]_{\theta}\\ &= \theta + \alpha \cdot \Expectover{\mu^{\pi_\theta}_{s,a}}{\big[\nabla_\theta \log \pi_\theta (a|s)\big]_{\theta} {Q}^{\pi_{\theta}}_{R,h}(s,a)} ,
\end{align*}
\looseness-1where ${Q}^{\pi_{\theta}}_{R,h}(s,a) = \Expect{\sum_{t=0}^h \gamma^t R(s_t, a_t) \big| s_0=s, a_0=a, T, \pi_{\theta}}$ is the $h$-depth $Q$-value with respect to $R$, and $\alpha$ is the learning rate. Furthermore, we assume that $L$ uses a tabular representation, where $\theta \in \mathbb{R}^{\abs{\mathcal{S}} \cdot \abs{\mathcal{A}}}$, and a softmax policy parameterization given by $\pi_\theta(a|s) := \frac{\exp(\theta(s, a))}{\sum_{b} \exp(\theta(s, b))}, \forall s \in \mathcal{S}, a \in \mathcal{A}$. For this $L$, the following proposition provides an intuitive form of the gradient of $I_L$ in Eq.~\eqref{eq:intuitive-IR-bi-level}.

\begin{proposition}
\label{prop:intuitive-grad}
The gradient of the informativeness criterion in Eq.~\eqref{eq:intuitive-IR-bi-level} for the simplified learning algorithm $L$ with $h$-depth planning described above takes the following form:
\begin{align*}
& \nabla_\phi I_L(R_\phi \mid \overline{R}, \pi^T, \pi^L) ~\approx~ \\
& \alpha \cdot \nabla_\phi \mathbb{E}_{\mu^{\pi^L}_{s,a}}\bigg[ \mu^{\pi^T}_s \cdot \pi^L(a|s) \cdot \big(A^{\pi^T}_{\overline{R}}(s,a) - {A}^{\pi^T}_{\overline{R}}(s, \pi^L(s))\big) \cdot {A^{\pi^L}_{R_\phi,h} (s,a)} \bigg] ,
\end{align*} 
where ${A}^{\pi^T}_{\overline{R}}(s, \pi^L(s)) = \Expectover{\pi^L(a'|s)}{{A}^{\pi^T}_{\overline{R}}(s, a')}$, and ${A}^{\pi^L}_{R_\phi,h} (s,a) = {Q}^{\pi^L}_{R_\phi,h} (s,a) - {V}^{\pi^L}_{R_\phi,h} (s)$.
\end{proposition}
\begin{proof}
We discuss key proof steps here and provide a more detailed proof in appendices of the paper. For the simple learning algorithm $L$ described above, we can write the derivative of the informativeness criterion in Eq.~\eqref{eq:intuitive-IR-bi-level} as follows:
\begin{align*}
\bss{\nabla_\phi I_L(R_\phi \mid \overline{R}, \pi^T, \pi^L)}_{\phi}&  \\
~\stackrel{(a)}{=}~& \bss{\nabla_\phi \theta^L_\textnormal{new} (\phi) \cdot \nabla_{\theta^L_\textnormal{new} (\phi)} J(\pi_{\theta^L_\textnormal{new}(\phi)}; \overline{R}, \pi^T)}_{\phi} \\ 
~\stackrel{(b)}{\approx}~& \bss{\nabla_\phi \theta^L_\textnormal{new} (\phi)}_{\phi} \cdot \bss{\nabla_{\theta} J(\pi_{\theta}; \overline{R}, \pi^T)}_{\theta^L} , 
\end{align*}
where the equality in $(a)$ is due to chain rule, and the approximation in $(b)$ assumes a smoothness condition of $\Big\lVert \bss{\nabla_{\theta} J(\pi_{\theta}; \overline{R}, \pi^T)}_{\theta^L_\textnormal{new}(\phi)} - \bss{\nabla_{\theta} J(\pi_{\theta}; \overline{R}, \pi^T)}_{\theta^L} \Big\rVert_2 \leq c \cdot \norm{\theta^L_\textnormal{new}(\phi) - \theta^L}_2$ for some $c > 0$. For the $L$ described above, we can obtain intuitive forms of the terms $\bss{\nabla_\phi \theta^L_\textnormal{new} (\phi)}_{\phi}$ and $\bss{\nabla_{\theta} J(\pi_{\theta}; \overline{R}, \pi^T)}_{\theta^L}$. For any $s \in \mathcal{S}, a \in \mathcal{A}$, let $\mathbf{1}_{s,a} \in \mathbb{R}^{\abs{\mathcal{S}} \cdot \abs{\mathcal{A}}}$ denote a vector with $1$ in the $(s,a)$-th entry and $0$ elsewhere. By using the meta-gradient derivations presented in~\cite{andrychowicz2016learning,santoro2016meta,nichol2018first}, we simplify the first term as follows:
\begin{align*}
\bss{\nabla_\phi \theta^L_\textnormal{new} (\phi)}_{\phi} ~=~& \alpha \cdot \Expectover{\mu_{s}^{\pi^L}}{\sum_{a}{\pi^L(a|s) \cdot \big[ \nabla_\phi {A}^{\pi^L}_{R_\phi,h} (s,a) \big]_{\phi} \cdot \mathbf{1}_{s,a}^\top}} .  
\end{align*}
Then, we simplify the second term as follows:
\begin{align*}
& \bss{\nabla_{\theta} J(\pi_{\theta}; \overline{R}, \pi^T)}_{\theta^L} \\
&\quad \quad \quad \quad ~=~ \Expectover{\mu_{s}^{\pi^T}}{\sum_{a}{\pi^L(a|s) \cdot \brr{{A}^{\pi^T}_{\overline{R}}(s,a) - {A}^{\pi^T}_{\overline{R}}(s, \pi^L(s))} \cdot \mathbf{1}_{s,a}}} .
\end{align*}
Taking the matrix product of two terms completes the proof.
\end{proof}

\subsection{Intuitive Formulation for Reward Informativeness $I_h(R)$}
\label{subsec:info-criterion-intuitive}

Based on Proposition~\ref{prop:intuitive-grad}, for the simple learning algorithm $L$ discussed in Section~\ref{subsec:info-criterion-proposal}, the informativeness criterion in Eq.~\eqref{eq:intuitive-IR-bi-level} can be written as follows: 
\begin{align*}
& I_L(R_\phi \mid \overline{R}, \pi^T, \pi^L) ~\approx~ \alpha \cdot \mathbb{E}_{\mu^{\pi^L}_{s,a}}\bigg[ \mu^{\pi^T}_s \cdot \pi^L(a|s) \\
& \qquad \qquad \qquad \qquad \cdot \big(A^{\pi^T}_{\overline{R}}(s,a) - {A}^{\pi^T}_{\overline{R}}(s, \pi^L(s))\big) \cdot {A^{\pi^L}_{R_\phi,h} (s,a)} \bigg] + \kappa ,    
\end{align*}
for some $\kappa \in \mathbb{R}$. By dropping the constant terms $\alpha$ and $\kappa$, we define the following intuitive informativeness criterion:
\begin{align}
&I_h(R_\phi \mid \overline{R}, \pi^T, \pi^L) ~:=~ \nonumber \\  &\mathbb{E}_{\mu^{\pi^L}_{s,a}}\bigg[ \mu^{\pi^T}_s \cdot \pi^L(a|s) \cdot \big(A^{\pi^T}_{\overline{R}}(s,a) - {A}^{\pi^T}_{\overline{R}}(s, \pi^L(s))\big) \cdot {A^{\pi^L}_{R_\phi,h} (s,a)} \bigg] .
\label{eq:intuitive-IR-final}
\end{align}
The above criterion doesn't require the knowledge of the learning algorithm $L$ and only relies on $\pi^L$, $\overline{R}$, and $\pi^T$. Therefore, it serves as a generic informativeness measure that can be used to evaluate the usefulness of reward functions for a range of limited-capacity learners, specifically those with different $h$-horizon planning budgets. In practice, we use the criterion $I_h$ with $h=1$. In this case, the criterion simplifies to the following form:
\vspace{-1.5mm}
\begin{align*}
&I_{h=1}(R_\phi \mid \overline{R}, \pi^T, \pi^L) ~:=~  \mathbb{E}_{\mu^{\pi^L}_{s,a}}\bigg[ \mu^{\pi^T}_s \cdot \pi^L(a|s) \\
& \quad \quad \quad \cdot \big( A^{\pi^T}_{\overline{R}}(s,a)- {A}^{\pi^T}_{\overline{R}}(s, \pi^L(s)) \big) \cdot \big(R_\phi (s,a) - R_\phi (s,\pi^L(s))\big) \bigg] , 
\vspace{-1.5mm}
\end{align*}
where $R_\phi (s,\pi^L(s)) = \Expectover{\pi^L(b|s)}{R_\phi (s,b)}$. Intuitively, this criterion measures the alignment of a reward function $R_\phi$ with better actions according to policy $\pi^T$, and how well it boosts the reward values for these actions in each state.

\vspace{-1mm}
\subsection{Using $I_h(R)$ in \AlgOurs{} Framework}
\label{subsec:info-criterion-expert}

\looseness-1Next, we will use the informativeness criterion $I_h$ for designing reward functions to accelerate the training process of a learning agent within the \AlgOurs{} framework. Specifically, we use $I_h$ in place of $I_L$ to address the reward design problem formulated in Eq.~\eqref{eq:reward-design-problem-orig}:
\begin{equation}
\max_{R_\phi \in \mathcal{R}} I_h(R_\phi \mid \overline{R}, \pi^T, \pi^L_{k-1}) ,
\label{eq:reward-design-problem}
\end{equation}
where the set $\mathcal{R}$ captures the additional constraints discussed in Section~\ref{sec:expadard-formal} (e.g., $\mathcal{R} = \mathcal{R}_\textnormal{inv} \cap \mathcal{R}_\textnormal{str}$). In Section~\ref{sec:evaluation}, we will implement \AlgOurs{} framework with two types of structural constraints and design adaptive reward functions for different learners; below, we theoretically showcase the utility of using $I_h$ by analyzing the improvement in the convergence in a simplified setting.

More concretely, we present a theoretical analysis of the reward design problem formulated in Eq.~\eqref{eq:reward-design-problem} without structural constraints and in a simplified setting to illustrate how this informativeness criterion for adaptive reward shaping can substantially improve the agent's convergence speed toward the target policy. For our theoretical analysis, we consider a finite MDP $M$, with the target policy $\pi^T$ being an optimal policy for this MDP. We use a tabular representation for the reward, i.e., $\phi \in \mathbb{R}^{\abs{\mathcal{S}} \cdot \abs{\mathcal{A}}}$. We consider a constraint set $\mathcal{R} = \bcc{R: \abs{R\brr{s,a}} \leq R_\mathrm{max}, \forall s \in \mathcal{S}, a \in \mathcal{A}}$. Additionally, we use the informativeness criterion in Eq.~\eqref{eq:intuitive-IR-final} with $h=1$, i.e., $I_{h=1}(R_\phi \mid \overline{R}, \pi^T, \pi^L)$. For the policy, we also use a tabular representation, i.e., $\theta \in \mathbb{R}^{\abs{\mathcal{S}} \cdot \abs{\mathcal{A}}}$. We use a greedy (policy iteration style) learning algorithm $L$ that first learns the $h$-step action-value function $Q_{R_k, h}^{\pi^L_{k-1}}$ w.r.t. current reward $R_k$ and updates the policy by selecting actions greedily based on the value function, i.e., $\pi^L_k (s) \gets \argmax_a Q_{R_k, h}^{\pi^L_{k-1}} (s,a)$ with random tie-breaking. In particular, we consider a learner with $h=1$, i.e., we have $\pi^L_k (s) \gets \argmax_a R_k (s,a)$. For the above setting, the following theorem provides a convergence guarantee for Algorithm~\ref{alg:expert-driven-adaptive-reward-design}.
\begin{theorem}
\label{thm:shaping-complexity}
Consider Algorithm~\ref{alg:expert-driven-adaptive-reward-design} with inputs $\pi^T$, $L$, $I_h$, and $\mathcal{R}$  as described above. We define a policy $\pi^{T, \textnormal{Adv}}$ induced by the advantage function of the target policy $\pi^T$ (w.r.t. $\overline{R}$) as follows: $\pi^{T, \textnormal{Adv}} (s) \gets \argmax_a A_{\overline{R}}^{\pi^T} (s,a)$ with random tie-breaking. Then, the learner's policy $\pi^L_k$ will converge to the policy $\pi^{T, \textnormal{Adv}}$ in $\mathcal{O}(\abs{\mathcal{A}})$ iterations.
\end{theorem}
\looseness-1Proof and additional details are provided in appendices of the paper. We note that the target policy $\pi^T$ does not need to be optimal for better convergence, and the results also hold with a sufficiently good (weak) target policy $\pi^{\widetilde{T}}$ s.t. $\pi^{\widetilde{T}, \textnormal{Adv}}$ is near-optimal.

\section{Experimental Evaluation}
\label{sec:evaluation}

In this section, we evaluate our expert-driven explicable and adaptive reward design framework, \AlgOurs{}, on two environments:  \envfourroom{} (Section~\ref{sec:evaluation:envfourroom}) and \envlinekey{} (Section~\ref{sec:evaluation:envlinekey}).
\envfourroom{} corresponds to a navigation task in a grid-world where the agent has to learn a policy to quickly reach the goal location in one of four rooms, starting from an initial location. Even though this environment has small state and action spaces, it provides a rich problem setting to validate different reward design techniques. In fact, variants of \envfourroom{} have been used in the literature~\cite{DBLP:conf/icml/McGovernB01,DBLP:conf/icml/SimsekWB05,grzes2008plan,DBLP:conf/aaai/AsmuthLZ08,DBLP:conf/atal/JamesS09,DBLP:conf/atal/0002CP19,jiang2021temporal_AAAI,devidze2021explicable, devidze2022exploration}. \envlinekey{} corresponds to a navigation task in a one-dimensional space where the agent has to first pick the  key and then reach the goal. The agent's location is represented as a node in a long chain. This environment is inspired by variants of navigation tasks in the literature where an agent needs to perform subtasks~\cite{DBLP:conf/icml/NgHR99,DBLP:conf/icml/RaileanuDSF18,devidze2021explicable, devidze2022exploration}. Both the \envfourroom{} and \envlinekey{} environments have sparse and delayed rewards, which pose a challenge for learning optimal behavior.

\vspace{-1mm}
\subsection{Evaluation on \envfourroom{}}
\label{sec:evaluation:envfourroom}

\textbf{\envfourroom{} (Figure~\ref{fig:envfourroom}).} This environment is based on the work of \cite{devidze2021explicable} that also serves as a baseline technique. The environment is represented as an MDP with $\mathcal{S}$ states corresponding to cells in a grid-world with the ``blue-circle'' indicating the agent's initial location. The goal (``green-star'') is located at the top-right corner cell. Agent can take four actions given by  $\mathcal{A} := \{\textnormal{``up''},  \textnormal{``left''}, \textnormal{``down''}, \textnormal{``right''}\}$. An action takes the agent to the neighbouring cell represented by the direction of the action; however, if there is a wall (``brown-segment''), the agent stays at the current location. There are also a few terminal walls (``thick-red-segment'') that terminate the episode, located at the bottom-left corner cell, where ``left'' and ``down'' actions terminate the episode; at the top-right corner cell, ``right'' action terminates. The agent gets a reward of $R_{\mathrm{max}}$ after it has navigated to the goal and then takes a ``right'' action (i.e., only one state-action pair has a reward); note that this action also terminates the episode. The reward is 0 for all other state-action pairs.
Furthermore, when an agent takes an action $a \in \mathcal{A}$, there is $p_{\textnormal{rand}}=0.05$ probability that an action $a' \in \mathcal{A} \setminus \{a\}$ will be executed. The environment-specific parameters are as follows: $R_{\mathrm{max}}=10$, $\gamma=0.95$, and the environment resets after a horizon of $H=30$ steps.

\looseness-1\textbf{Reward structure.} In this environment, we consider a configuration of nine $3\times3$ grids along with a single $1\times1$ grid representing the goal state, as visually depicted in Figure~\ref{fig:envfourroom.feature}. To effectively represent the state space, we employ a state abstraction function denoted as $\psi: \mathcal{S} \to \{0,1\}^{10}$. For each state $s \in \mathcal{S}$, the $i$-th entry of $\psi(s)$ is set to $1$ if $s$ resides in the $i$-th grid, and $0$ otherwise. Building upon this state abstraction, we introduce a feature representation function, $f: \mathcal{S} \times \mathcal{A} \to \{0,1\}^{10 \cdot \abs{\mathcal{A}}}$, defined as follows: $f(s,a)_{(\cdot,a)} = \psi(s)$, and $f(s,a)_{(\cdot,a')} = \mathbf{0}, \forall a' \neq a$. Here, for any vector $v \in \{0,1\}^{10 \cdot \abs{\mathcal{A}}}$, we use the notation $v_{(i,a)}$ to refer to the $(i,a)$-th entry of the vector. Finally, we establish the set $\mathcal{R}_\textnormal{str} = \{R_\phi: \mathcal{S} \times \mathcal{A} \to \mathbb{R} \text{ where } \phi \in \mathbb{R}^d\}$, where $R_\phi(s,a) = \ipp{\phi}{f(s,a)}$. Further, we define $\mathcal{R} := \mathcal{R}_\textnormal{inv} \cap \mathcal{R}_\textnormal{str}$ as discussed in Section~\ref{sec:expadard-formal}. We note that $\overline{R} \in \mathcal{R}$.

\looseness-1\textbf{Evaluation setup.} We conduct our experiments with a tabular REINFORCE agent~\cite{sutton2018reinforcement}, and employ an optimal policy under the underlying reward function $\overline{R}$ as the target policy $\pi^T$. Algorithm~\ref{alg:expert-driven-adaptive-reward-design} provides a sketch of the overall training process and shows how the agent's training interleaves with the expert-driven reward design process. Specifically, during training, the agent receives rewards based on the designed reward function ${R}$; the performance is always evaluated w.r.t. $\overline{R}$ (also reported in the plots). 
In our experiments, we considered two settings to systematically evaluate the utility of adaptive reward design: (i) a single learner with a uniformly random initial policy (where each action is taken with a probability of $0.25$) and (ii) a diverse group of learners, each with distinct initial policies. To generate a collection of distinctive initial policies, we introduced modifications to a uniformly random policy. These modifications were designed to incorporate a $0.5$ probability of the agent selecting suboptimal actions when encountering various ``gate-states'' (i.e., states with openings for navigation to other rooms). In our evaluation, we included five such unique initial policies. 

\looseness-1\textbf{Techniques evaluated.} We evaluate the effectiveness of the following reward design techniques: 
\begin{enumerate}[label={(\roman*)},leftmargin=16pt]
    \item ${R}^{\orig} := \overline{R}$ is a default baseline without any reward design.
    \item ${R}^{\Invariance}$ is obtained via solving the optimization problem in Eq.~\eqref{eq:reward-design-problem} with the substitution of $I_h$ by a constant. This technique does not involve explicitly maximizing any reward informativeness during the optimization process.
    \item ${R}^{\EXPRD}$  is obtained via solving the optimization problem proposed in \cite{devidze2021explicable}. This optimization problem is equivalent to Eq.~\eqref{eq:reward-design-problem}, with the substitution of $I_h$ by a non-adaptive informativeness criterion. We have employed the hyperparameters consistent with those provided in their work.
    \item \looseness-1${R}_{k}^{\AlgOurs}$ is based on our framework \AlgOurs{} and obtained via solving the optimization problem in Eq.~\eqref{eq:reward-design-problem}. For stability of the learning process, we update the policy more frequently than the reward as typically considered in the literature~\cite{zheng2018learning,memarian2021self,devidze2022exploration} -- we provide additional details in appendices of the paper.
\end{enumerate}

\begin{figure*}[!htb]
\centering
    \begin{subfigure}[b]{.24\textwidth}
    \centering
    {
          \includegraphics[height=3.5cm]{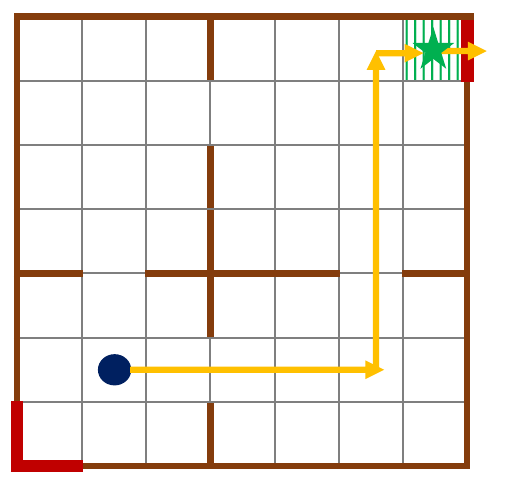}
  \caption{\envfourroom{}: Environment}
      \label{fig:envfourroom}
    }
    \end{subfigure}
    \begin{subfigure}[b]{.24\textwidth}
    \centering
    {
          \includegraphics[height=3.5cm]{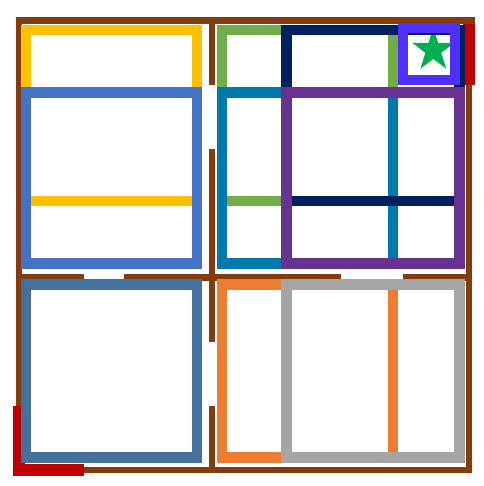}
  \caption{\envfourroom{}: Abstracted features}
      \label{fig:envfourroom.feature}
    }
    \end{subfigure}
    \begin{subfigure}[b]{.24\textwidth}
    \centering
    {
        \includegraphics[width=0.89\textwidth]{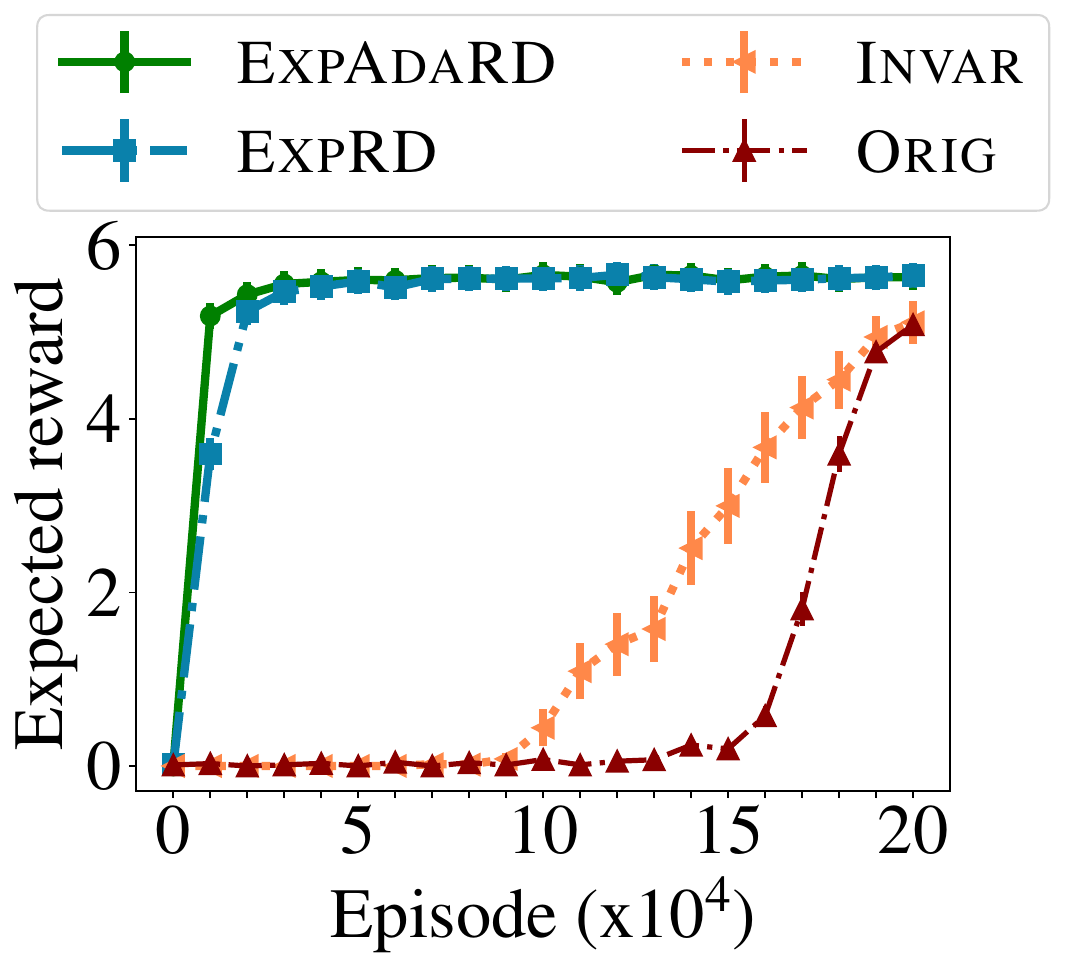}
                \caption{\envfourroom{}: Single learner setting}
        \label{fig:convergence.results.room.reinforce1}
    }
    \end{subfigure}
    \begin{subfigure}[b]{.24\textwidth}
    \centering
    {
        \includegraphics[width=0.89\textwidth]{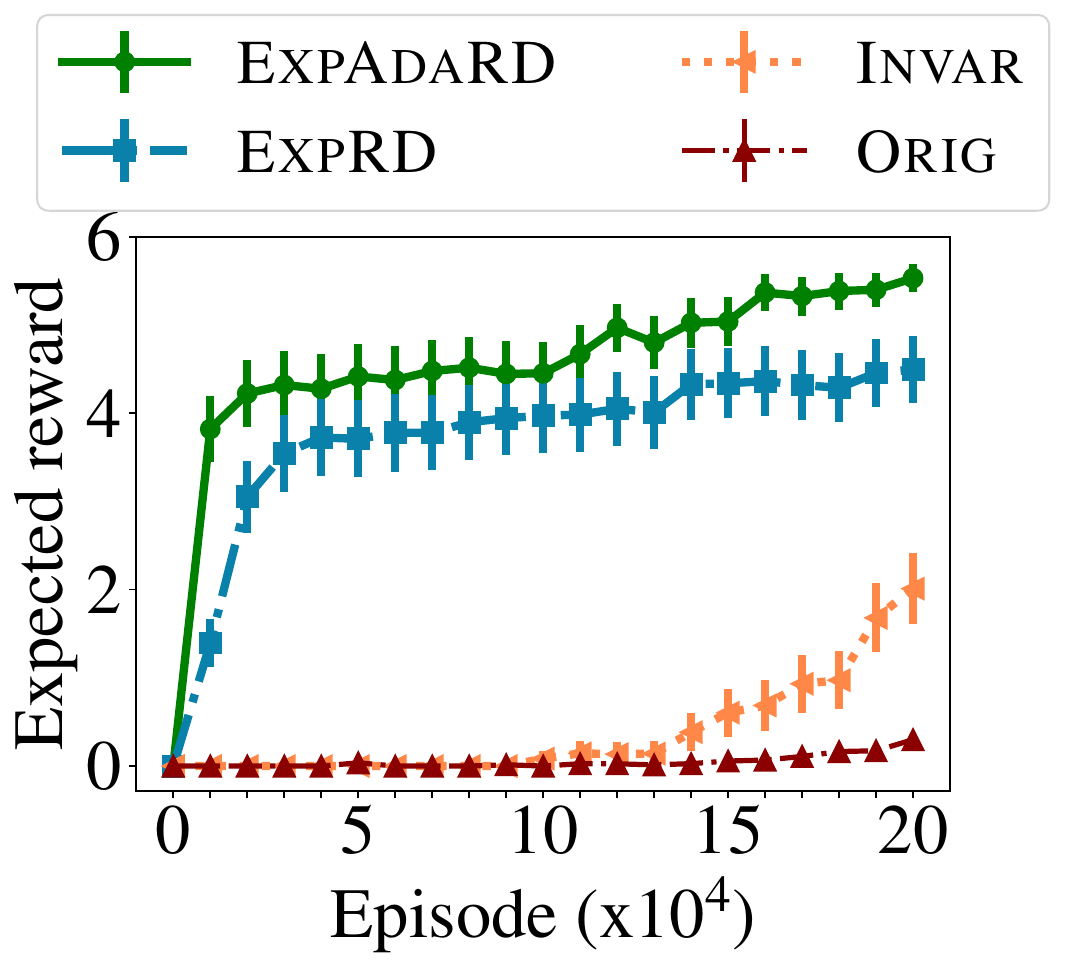}         
        \caption{\envfourroom{}: Diverse learners setting}
        \label{fig:convergence.results.room.reinforce2}
    }
    \end{subfigure}
    \vspace{-2mm}
    \caption{\looseness-1Results for \envfourroom{}. \textbf{(a)} shows the environment. \textbf{(b)} shows the abstracted feature space used for the representation of designed reward functions as a structural constraint. \textbf{(c)} shows results for the setting with a single learner. \textbf{(d)} shows results for the setting with a diverse group of learners with different initial policies. \AlgOurs{} designs adaptive reward functions w.r.t. the learner's current policies, whereas other techniques are agnostic to the learner's policy. See Section~\ref{sec:evaluation:envfourroom} for details.}
    \label{fig:convergence.results.room}
    \Description{Convergence results for the four-room environment.}
    \vspace{-1mm}
\end{figure*}
\begin{figure*}[t!]
\centering
    \begin{subfigure}[b]{.24\textwidth}
    \centering
    {
      \includegraphics[width=4.7cm]{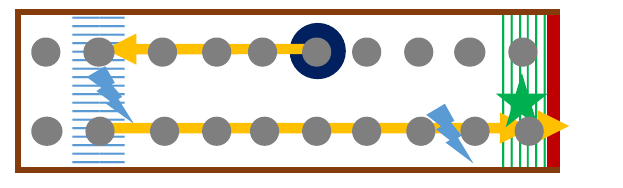}
      \vspace{8mm}
    \caption{\envlinekey{}: Environment} 
    \label{fig:envlinekey}
    }
    \end{subfigure}
    \begin{subfigure}[b]{.24\textwidth}
    \centering
    {  
        \includegraphics[trim={1.3cm 1.3cm 1.3cm 1.3cm}, clip, height=3.35cm]{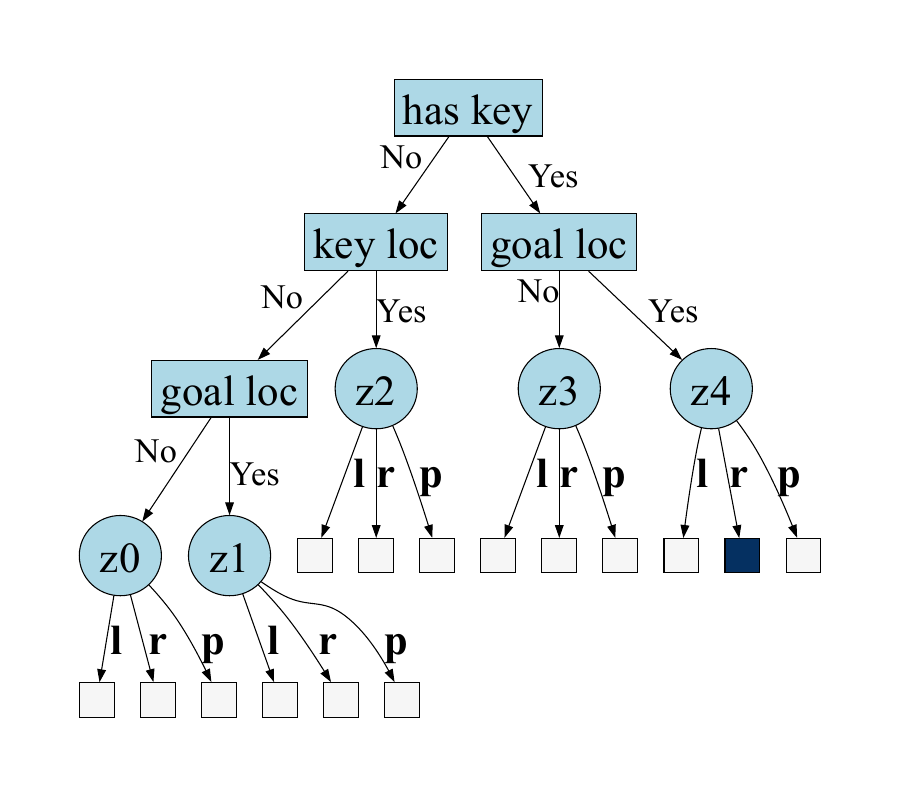}
    \caption{\envlinekey{}: Tree-based features} 
    \label{fig:envlinekey.tree.structure}
    }
    \end{subfigure}
    \begin{subfigure}[b]{.24\textwidth}
    \centering
    {
      \includegraphics[width=0.89\textwidth]{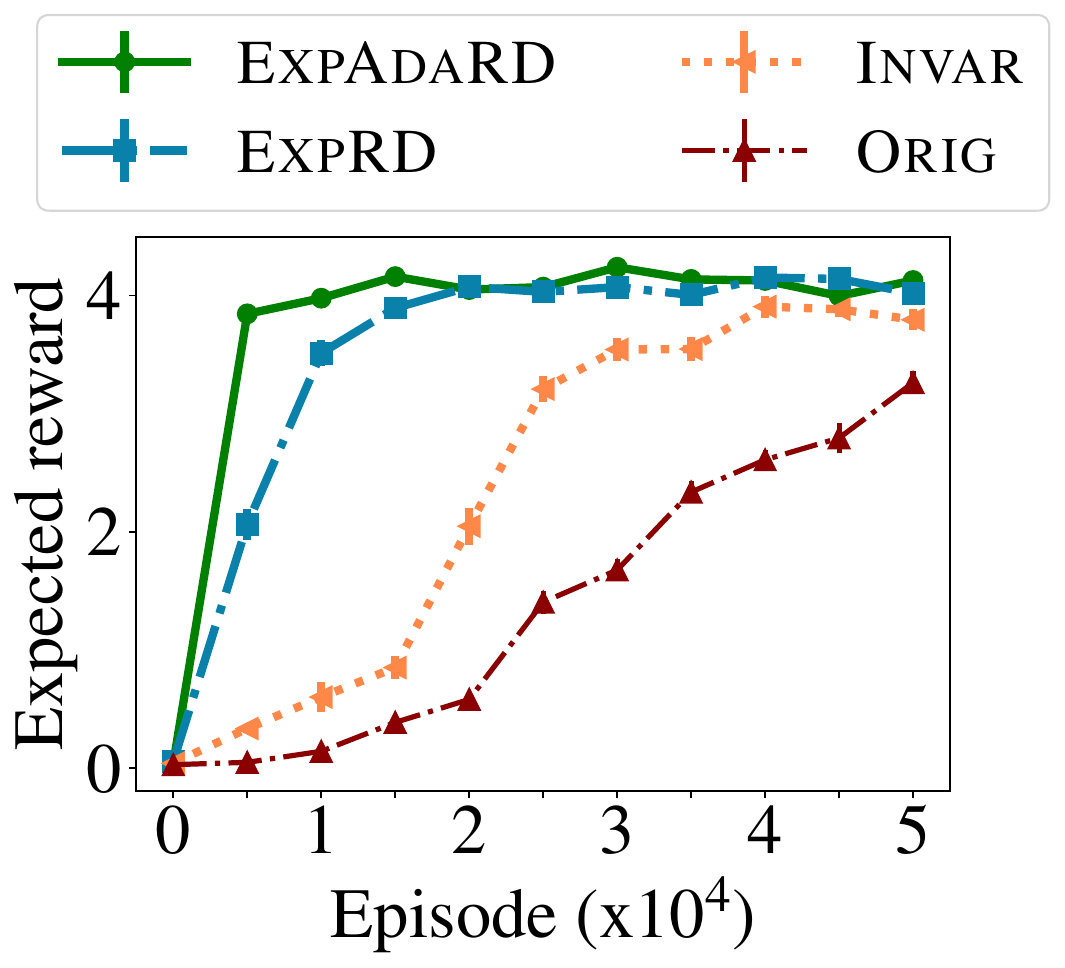}
         \caption{\envlinekey{}: Single learner setting}
        \label{fig:convergence.results.linekey.reinforce.1}
    }
    \end{subfigure}
    \begin{subfigure}[b]{.24\textwidth}
    \centering
    {
      \includegraphics[width=0.89\textwidth]{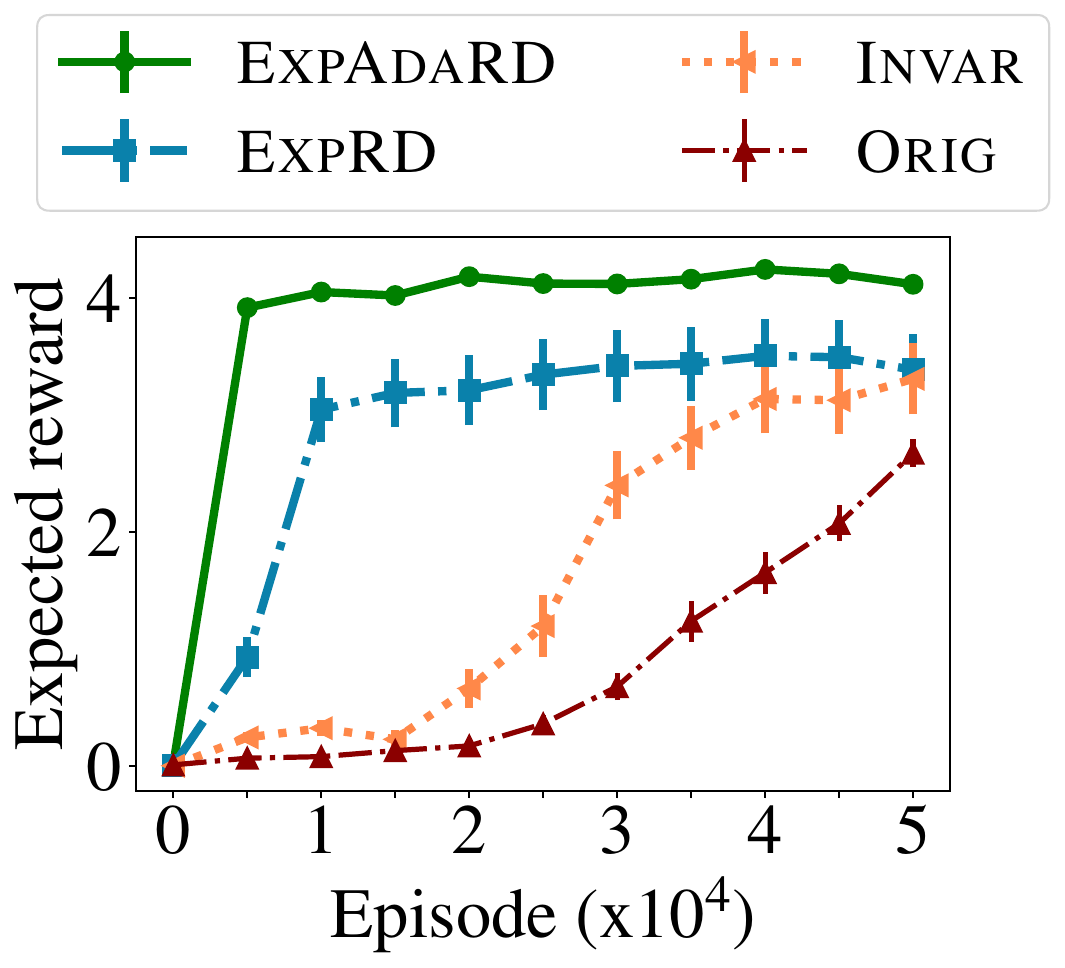}
         \caption{\envlinekey{}: Diverse learners setting}
        \label{fig:convergence.results.linekey.reinforce.2}
    }
    \end{subfigure}
    \vspace{-3mm}
    \caption{Results for \envlinekey{}. \textbf{(a)} shows the environment. \textbf{(b)} shows the tree-based feature space used for the representation of designed reward functions as a structural constraint. \textbf{(c)} shows results for the setting with a single learner. \textbf{(d)} shows results for the setting with a diverse group of learners with different initial policies. \AlgOurs{} designs adaptive reward functions w.r.t. the learner's current policies, whereas other techniques are agnostic to the learner's policy. See Section~\ref{sec:evaluation:envlinekey} for details.}
    \label{fig:convergence.results.linekey}
    \Description{Convergence results for the line-key environment.}
\end{figure*}

\begin{figure*}[!htb]
\begin{subfigure}[b]{.325\textwidth}
    \centering{
    \begin{minipage}[b]{\textwidth}
    \centering
    {
        \includegraphics[width=0.475\textwidth]{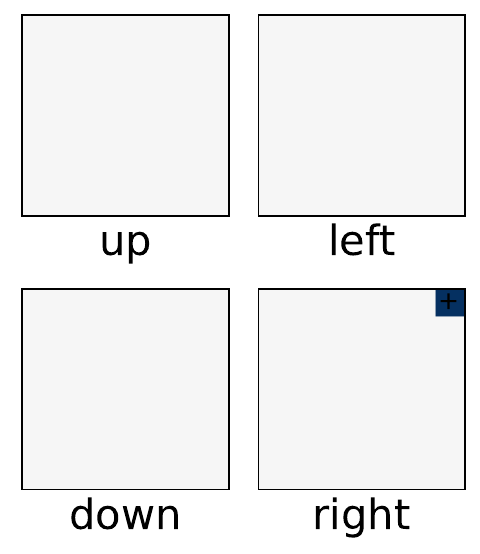}
    }
    \end{minipage}
    }
    \vspace{-6mm}
    \caption{\envfourroom{}: ${R}^{\orig}$}
    \label{fig:designed.rewards.room.Orig} 
\end{subfigure}
\begin{subfigure}[b]{.325\textwidth}
    \centering{
    \begin{minipage}[b]{\textwidth}
    \centering
    {
        \includegraphics[width=0.475\textwidth]{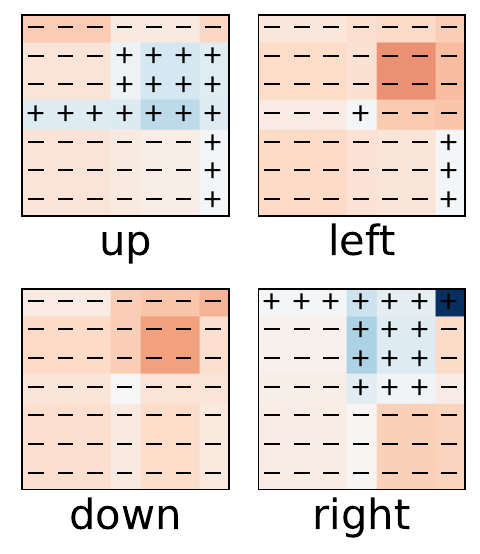}
    }
    \end{minipage} 
    }
    \vspace{-6mm}
    \caption{\envfourroom{}: ${R}^{\Invariance}$}
    \label{fig:designed.rewards.room.Invariance.L1}
\end{subfigure}
\vspace{0.5mm}
\begin{subfigure}[b]{0.325\textwidth}
    \centering{
    \begin{minipage}[b]{\textwidth}
    \centering
    {
        \includegraphics[width=0.475\textwidth]{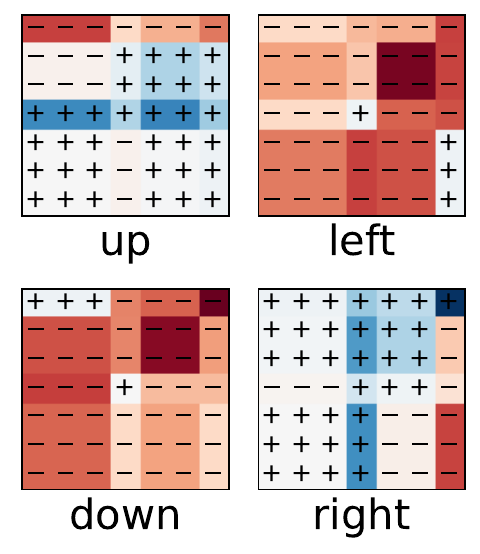}
    }
    \end{minipage}
    }
    \vspace{-6mm}
    \caption{\envfourroom{}: ${R}^{\EXPRD}$}
    \label{fig:designed.rewards.room.EXPRD}   
\end{subfigure}
\vspace{0.5mm}
\begin{subfigure}[b]{\textwidth}
    \centering{
    \begin{minipage}[b]{.19\textwidth}
    \centering
    {
        \includegraphics[width=0.8\textwidth]{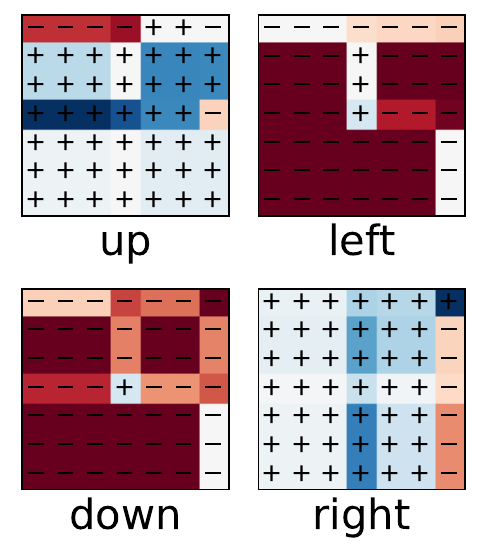}
    }
    \end{minipage}
    \begin{minipage}[b]{.19\textwidth}
    \centering
    {
        \includegraphics[width=0.8\textwidth]{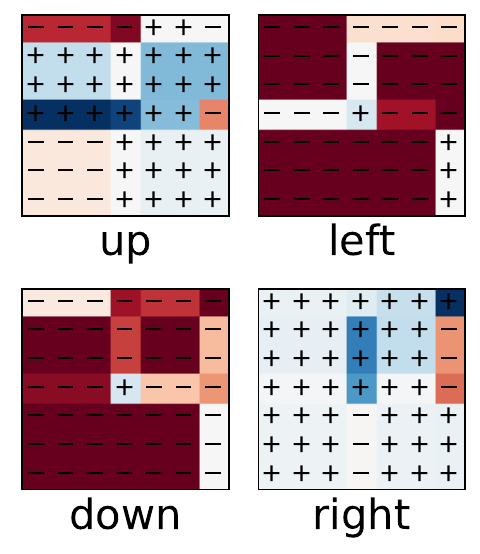}
    }
    \end{minipage}  
    \begin{minipage}[b]{.19\textwidth}
    \centering
    {
        \includegraphics[width=0.8\textwidth]{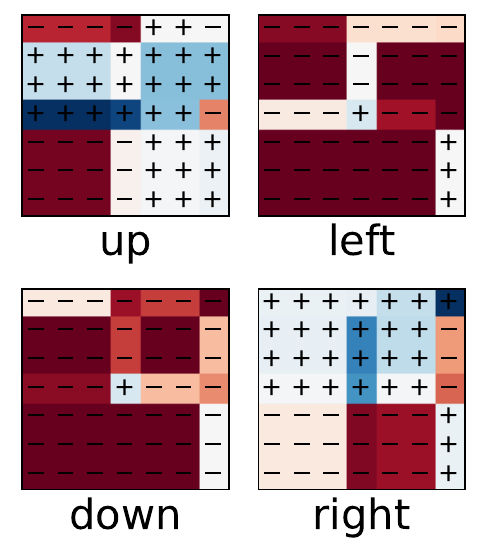}
    }
    \end{minipage} 
    \begin{minipage}[b]{.19\textwidth}
    \centering
    {
        \includegraphics[width=0.8\textwidth]{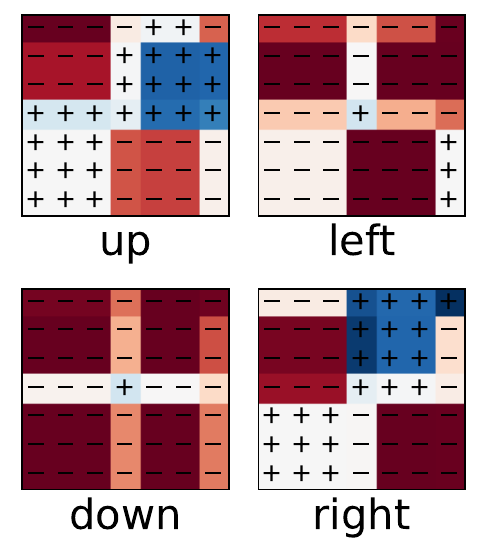}
    }
    \end{minipage}
    \begin{minipage}[b]{.19\textwidth}
    \centering
    {
        \includegraphics[width=0.8\textwidth]{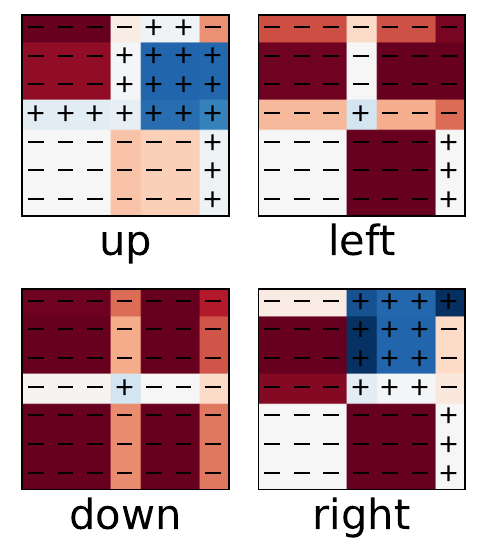}
    }
    \end{minipage}
    }
    \vspace{-2mm}
    \caption{\envfourroom{}: ${R}^{\AlgOurs}$ for learner $1$ at $k=1000, 2000, 3000, 100000, \text{and } 200000$ episodes.}
    \label{fig:designed.rewards.room.Ada_TL.L1}
\end{subfigure}
\vspace{0.5mm}
\begin{subfigure}[b]{\textwidth}
    \centering{
    \begin{minipage}[b]{.19\textwidth}
    \centering
    {
        \includegraphics[width=0.8\textwidth]{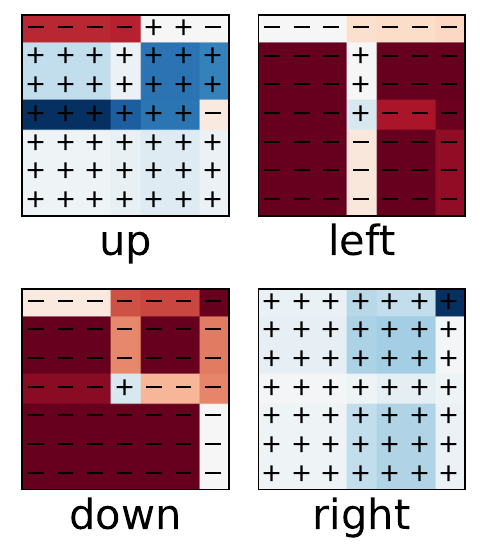}
    }
    \end{minipage}
    \begin{minipage}[b]{.19\textwidth}
    \centering
    {
        \includegraphics[width=0.8\textwidth]{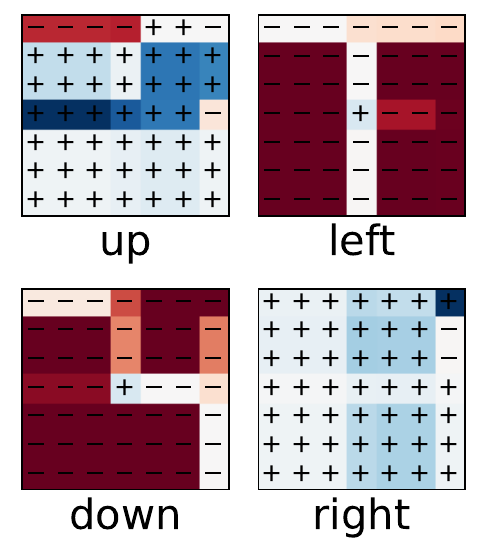}
    }
    \end{minipage}
    \begin{minipage}[b]{.19\textwidth}
    \centering
    {
        \includegraphics[width=0.8\textwidth]{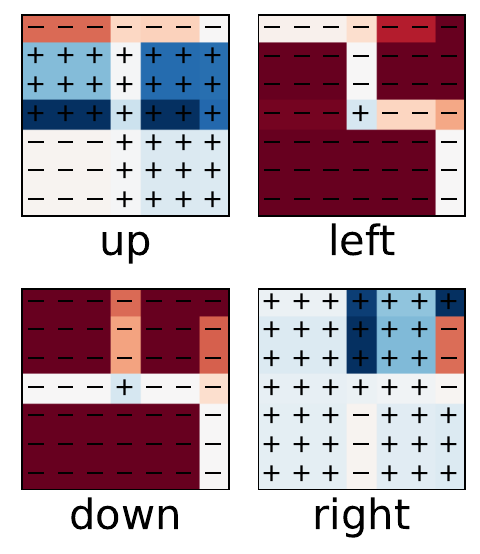}
    }
    \end{minipage}
    \begin{minipage}[b]{.19\textwidth}
    \centering
    {
        \includegraphics[width=0.8\textwidth]{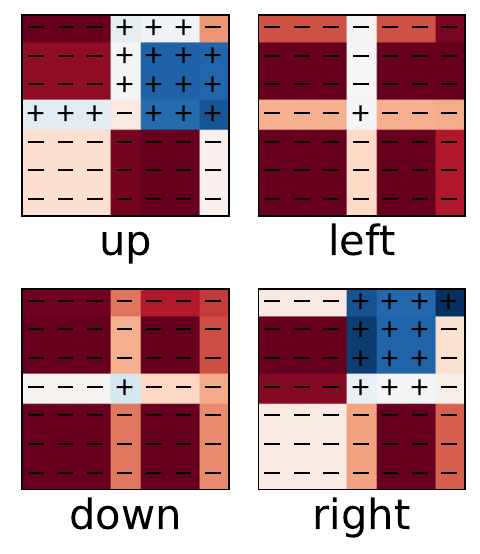}
    }
    \end{minipage}
    \begin{minipage}[b]{.19\textwidth}
    \centering
    {
        \includegraphics[width=0.8\textwidth]{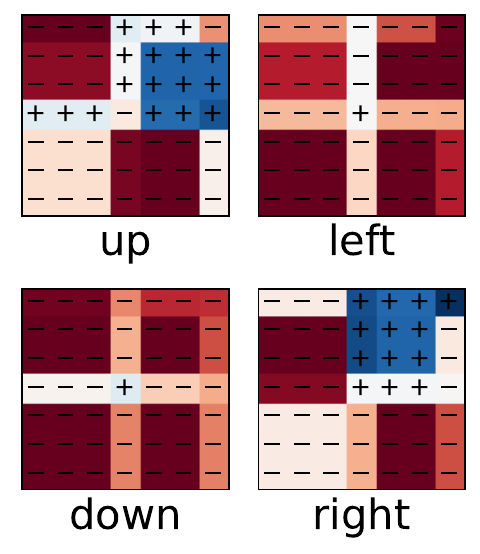}
    }
    \end{minipage}
    }
    \vspace{-2mm}
    \caption{\envfourroom{}: ${R}^{\AlgOurs}$ for learner $2$ at $k=1000, 2000, 3000, 100000, \text{and } 200000$ episodes.}
    \label{fig:designed.rewards.room.Ada_TL.L2} 
\end{subfigure}
\vspace{0.5mm}
\begin{subfigure}[b]{\textwidth}
    \centering{
    \begin{minipage}[b]{.19\textwidth}
    \centering
    {
        \includegraphics[width=0.8\textwidth]{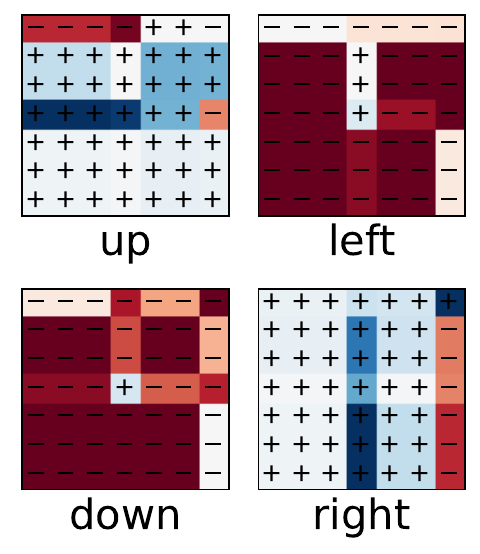}
    }
    \end{minipage}
    \begin{minipage}[b]{.19\textwidth}
    \centering
    {
        \includegraphics[width=0.8\textwidth]{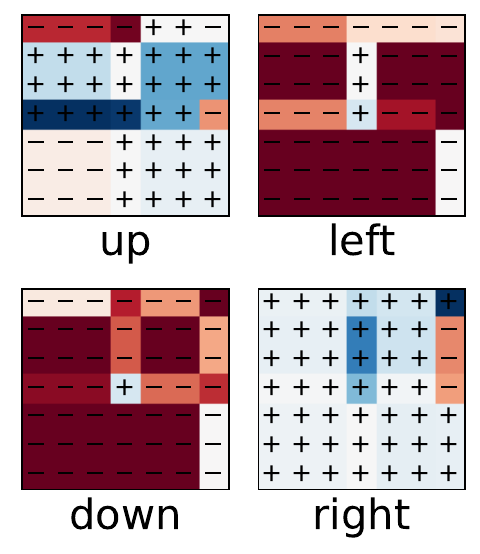}
    }
    \end{minipage}
    \begin{minipage}[b]{.19\textwidth}
    \centering
    {
        \includegraphics[width=0.8\textwidth]{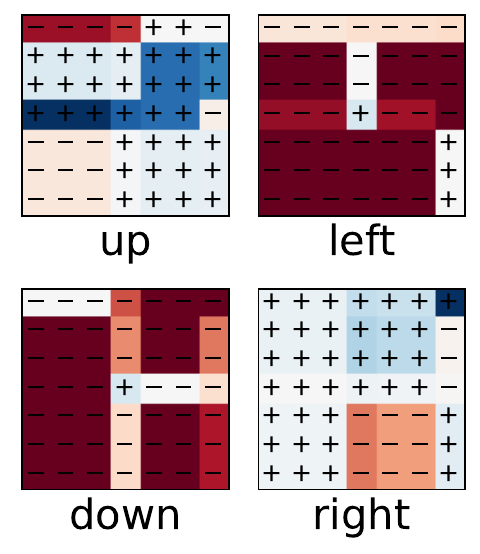}
    }
    \end{minipage}
    \begin{minipage}[b]{.19\textwidth}
    \centering
    {
        \includegraphics[width=0.8\textwidth]{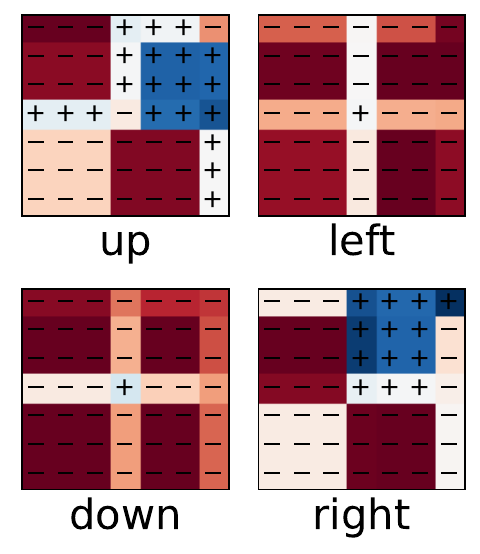}
    }
    \end{minipage}
    \begin{minipage}[b]{.19\textwidth}
    \centering
    {
        \includegraphics[width=0.8\textwidth]{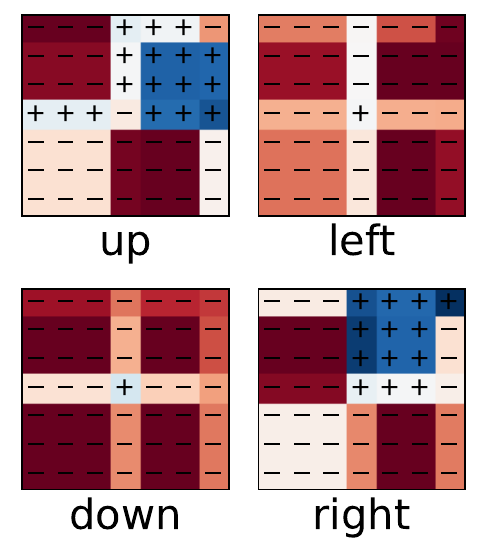}
    }
    \end{minipage}
    }
    \vspace{-2mm}
    \caption{\envfourroom{}: ${R}^{\AlgOurs}$ for learner $3$ at $k=1000, 2000, 3000, 100000, \text{and } 200000$ episodes.}
    \label{fig:designed.rewards.room.Ada_TL.L3} 
\end{subfigure}
    \vspace{-6.5mm}
    \caption{Visualization of reward functions designed by different techniques in the \envfourroom{} environment for all four actions $\{\textnormal{``up''}, \textnormal{``left''}, \textnormal{``down''}, \textnormal{``right''}\}$. \textbf{(a)} shows original reward function ${R}^{\orig}$. \textbf{(b)} shows reward function ${R}^{\Invariance}$. \textbf{(c)} shows reward function ${R}^{\EXPRD}$ designed by expert-driven non-adaptive reward design technique~\cite{devidze2021explicable}. \textbf{(d, e, f)} show reward functions ${R}^{\AlgOurs}$ designed by our framework \AlgOurs{} for three learners, each with its distinct initial policy, at different training episodes $k$. A negative reward is shown in Red color with the sign ``-'', a positive reward is shown in Blue color with the sign ``+'', and a zero reward is shown in white. The color intensity indicates the magnitude of the reward.}
    \label{fig:designed.rewards.room}
    \Description{Designed rewards for the four-room environment.}
    \vspace{-3.5mm}    
\end{figure*}
\begin{figure*}[t!]
\begin{subfigure}[b]{0.325\textwidth}
    \centering{
    \begin{minipage}[b]{.645\textwidth}
    \centering
    { 
        \includegraphics[trim={1.3cm 1.3cm 1.3cm 1.3cm}, clip, height=3.1cm]{figs/results/linekey/Orig_linekey_designed_rewards_with_pool_learner_1_step=100_tree.pdf}
    }
    \end{minipage}
    \begin{minipage}[b]{.333\textwidth}
    \centering
    {
        \includegraphics[height=3.2cm]{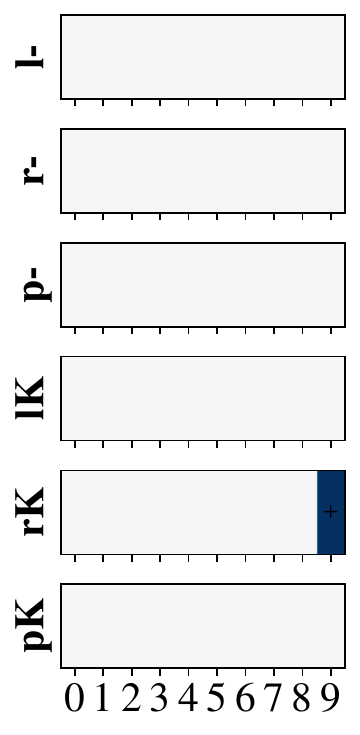}

    }
    \end{minipage}
    }
    \vspace{-5mm}    
    \caption{\envlinekey{}: ${R}^{\orig}$}
    \label{fig-app:designed.rewards.linekey.Orig}
\end{subfigure}
\begin{subfigure}[b]{0.325\textwidth}
    \centering{
    \begin{minipage}[b]{.645\textwidth}
    \centering
    { 
        \includegraphics[trim={1.3cm 1.3cm 1.3cm 1.3cm}, clip, height=3.1cm]{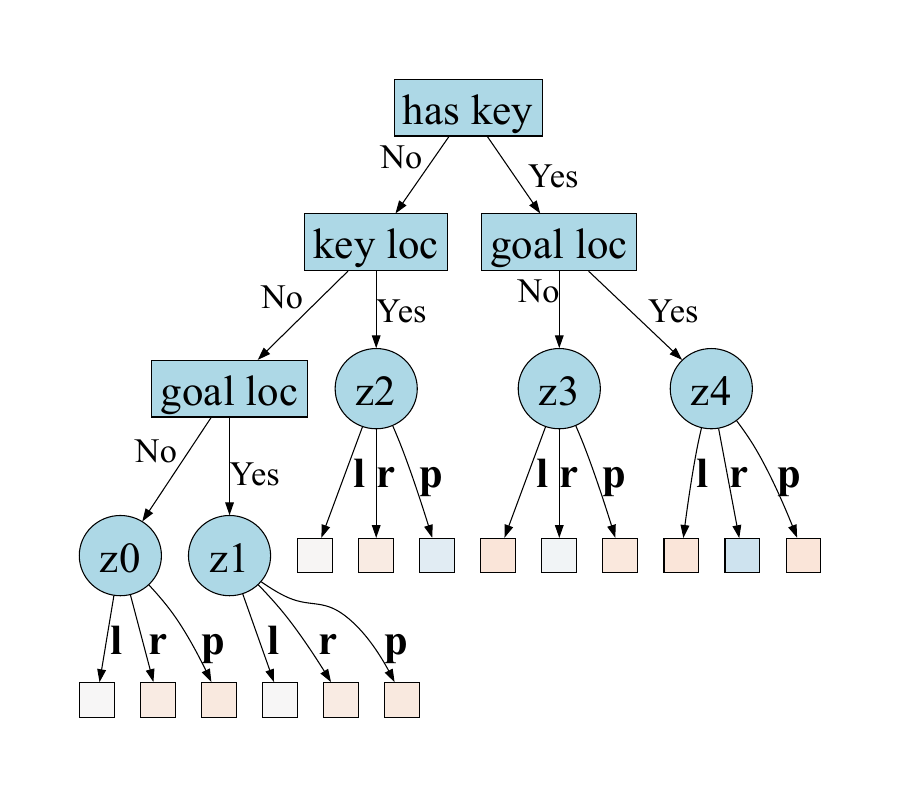}
    }
    \end{minipage}
    \begin{minipage}[b]{.333\textwidth}
    \centering
    {
        \includegraphics[height=3.2cm]{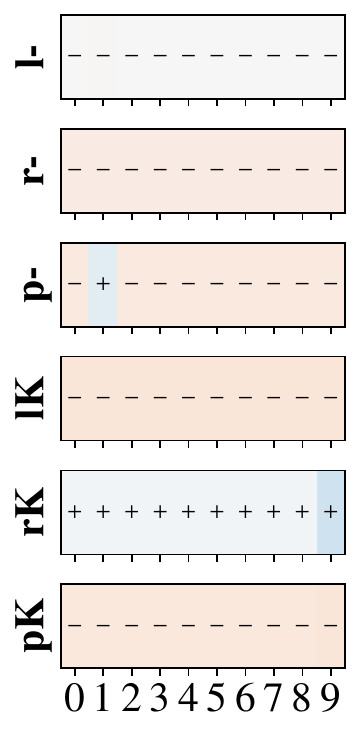}
    }
    \end{minipage}
    }
    \vspace{-5mm}    
    \caption{\envlinekey{}: ${R}^{\Invariance}$}
    \label{fig:designed.rewards.linekey.pool.Invariance}
\end{subfigure}
\vspace{2mm}
\begin{subfigure}[b]{0.325\textwidth}
    \centering{
    \begin{minipage}[b]{.645\textwidth}
    \centering
    { 
        \includegraphics[trim={1.3cm 1.3cm 1.3cm 1.3cm}, clip, height=3.1cm]{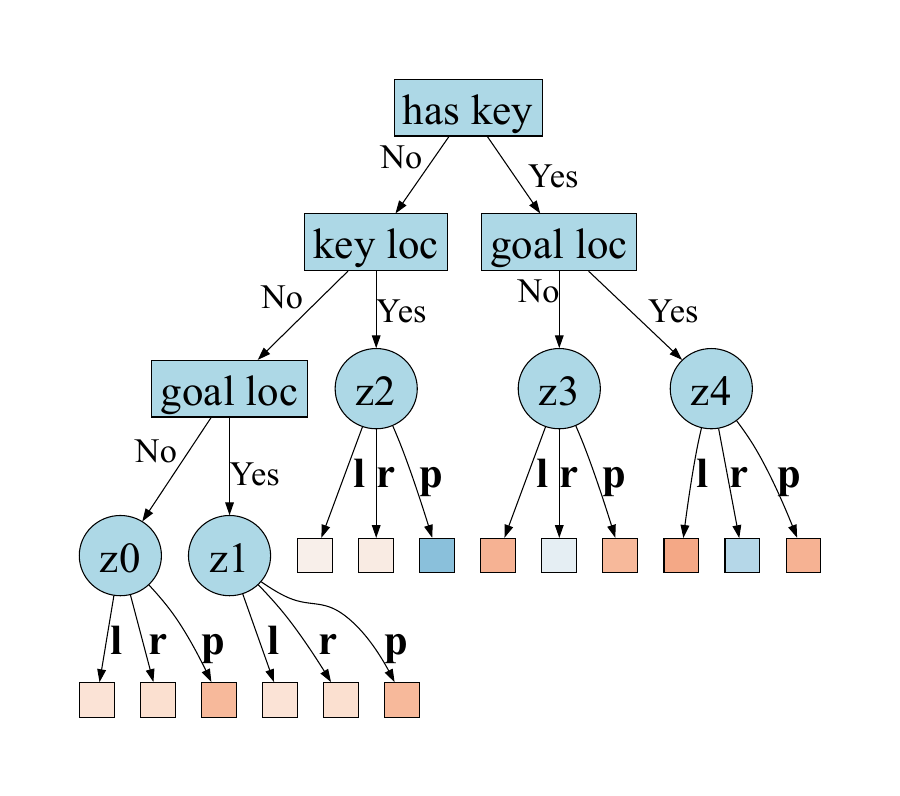}
    }
    \end{minipage}
    \begin{minipage}[b]{.333\textwidth}
    \centering
    {
        \includegraphics[height=3.2cm]{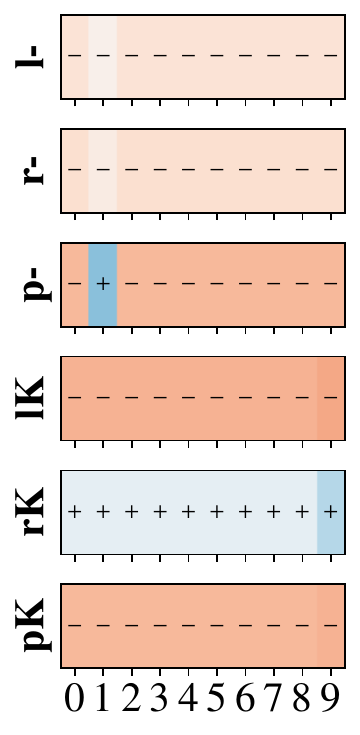}
    }
    \end{minipage}
    }
    \vspace{-5mm}
    \caption{\envlinekey{}: ${R}^{\EXPRD}$}
    \label{fig-app:designed.rewards.linekey.EXPRD}
\end{subfigure}
\vspace{2mm}
\begin{subfigure}[b]{\textwidth}
    \centering{
    \begin{minipage}[b]{.215\textwidth}
    \centering
    { 
        \includegraphics[trim={1.3cm 1.3cm 1.3cm 1.3cm}, clip, height=3.1cm]{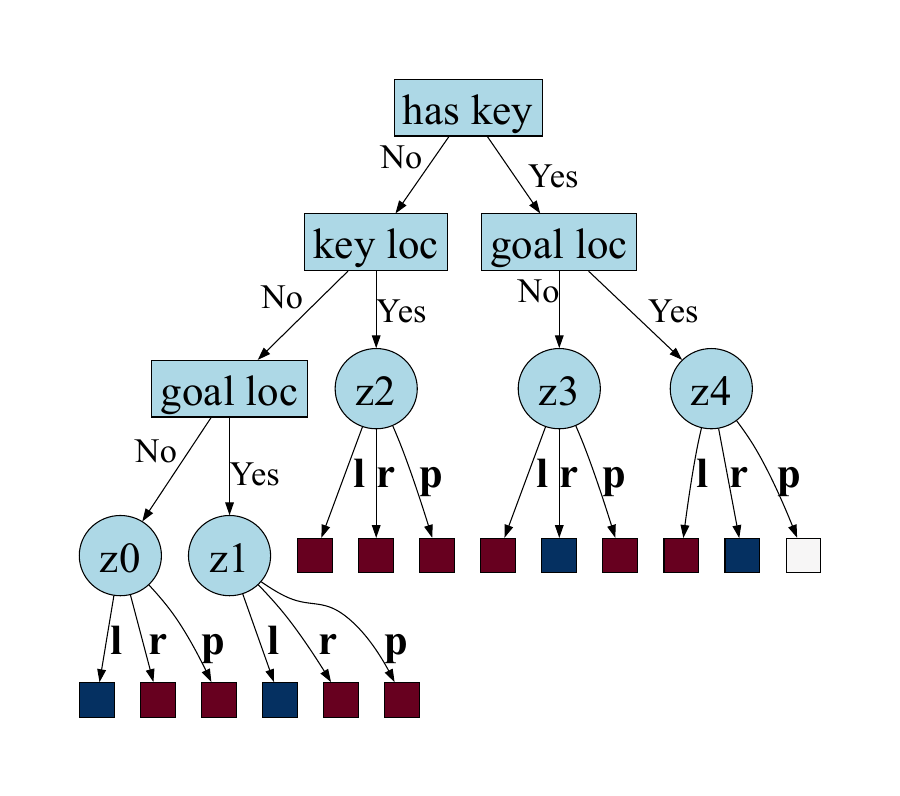}
    }
    \end{minipage}
    \begin{minipage}[b]{.111\textwidth}
    \centering
    {
        \includegraphics[height=3.2cm]{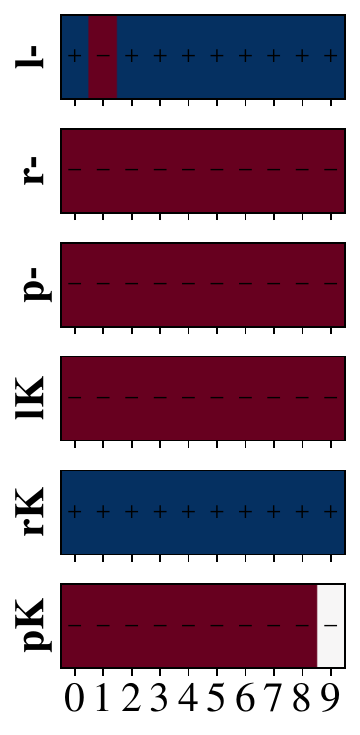}
    }
    \end{minipage}
    \begin{minipage}[b]{.215\textwidth}
    \centering
    {
        \includegraphics[trim={1.3cm 1.3cm 1.3cm 1.3cm}, clip, height=3.1cm]{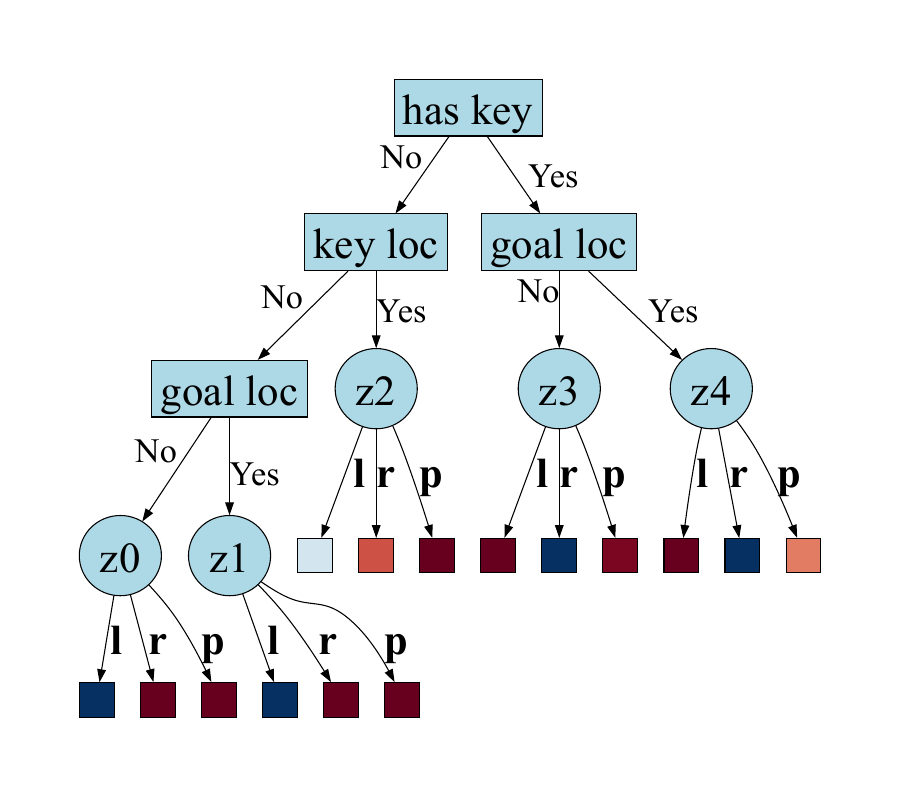}
    }
    \end{minipage}
    \begin{minipage}[b]{.111\textwidth}
    \centering
    {
        \includegraphics[height=3.2cm]{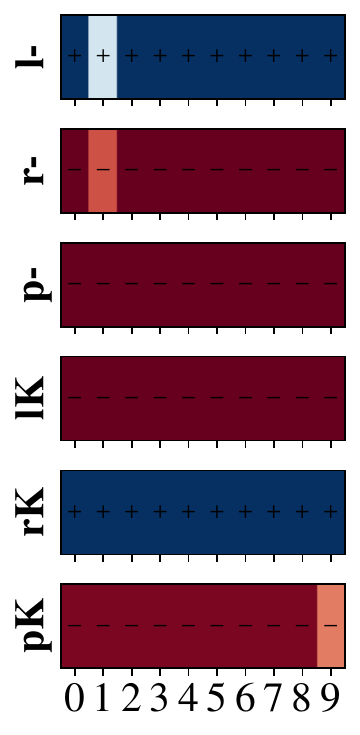}
    }
    \end{minipage}
    \begin{minipage}[b]{.215\textwidth}
    \centering
    {
        \includegraphics[trim={1.3cm 1.3cm 1.3cm 1.3cm}, clip, height=3.1cm]{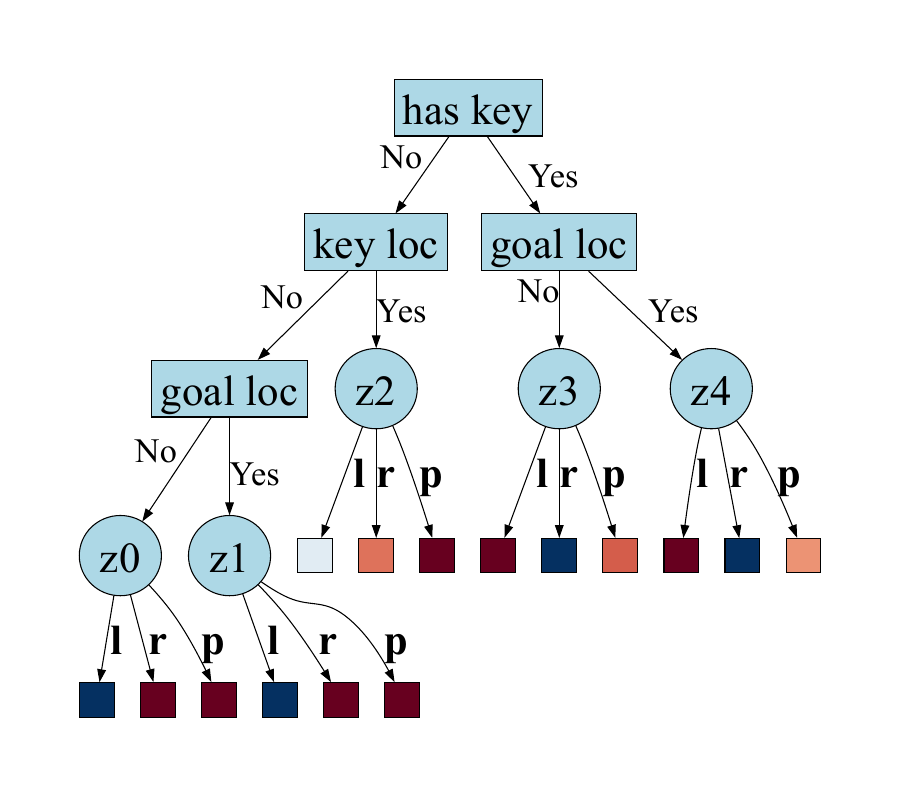}
    }
    \end{minipage}
    \begin{minipage}[b]{.111\textwidth}
    \centering
    {
        \includegraphics[height=3.2cm]{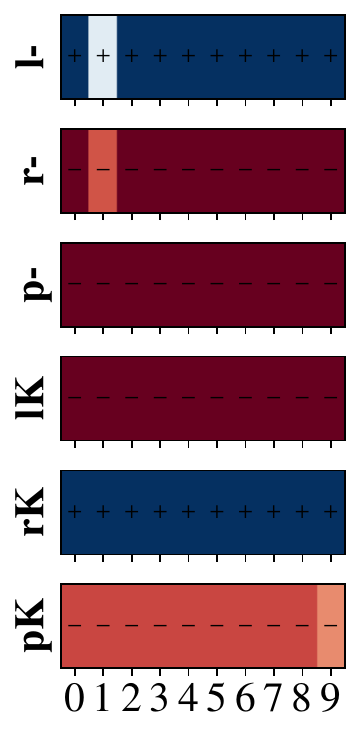}
    }
    \end{minipage}
    }
    \vspace{-4mm}     
    \caption{\envlinekey{}: ${R}^{\AlgOurs}$ for learner $1$  at $k=100, 30000, \text{and } 50000$ episodes.}
    \label{fig:designed.rewards.linekey.pool.Ada_TL.L1}
\end{subfigure}
\vspace{2mm}
\begin{subfigure}[b]{\textwidth}
    \centering{
    \begin{minipage}[b]{.215\textwidth}
    \centering
    { 
        \includegraphics[trim={1.3cm 1.3cm 1.3cm 1.3cm}, clip, height=3.1cm]{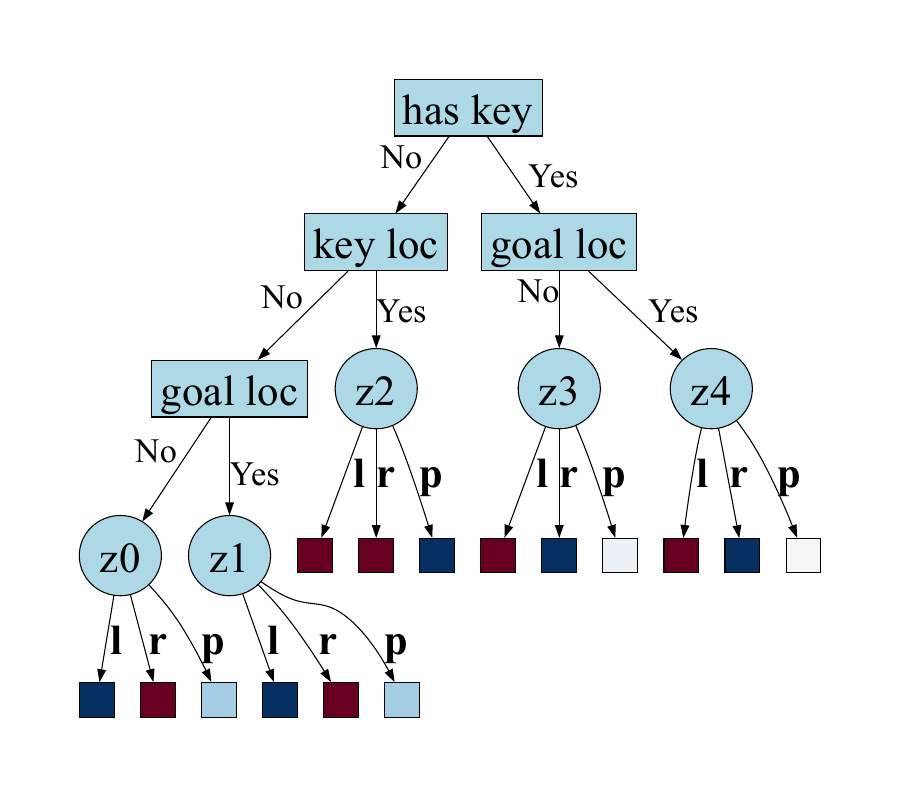}
    }
    \end{minipage}
    \begin{minipage}[b]{.111\textwidth}
    \centering
    {
        \includegraphics[height=3.2cm]{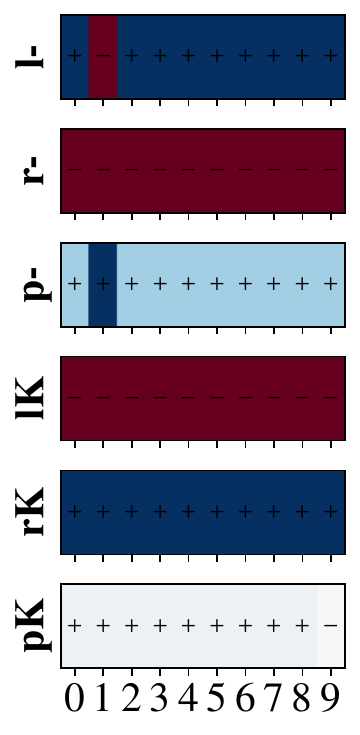}
    }
    \end{minipage}
    \begin{minipage}[b]{.215\textwidth}
    \centering
    {
        \includegraphics[trim={1.3cm 1.3cm 1.3cm 1.3cm}, clip, height=3.1cm]{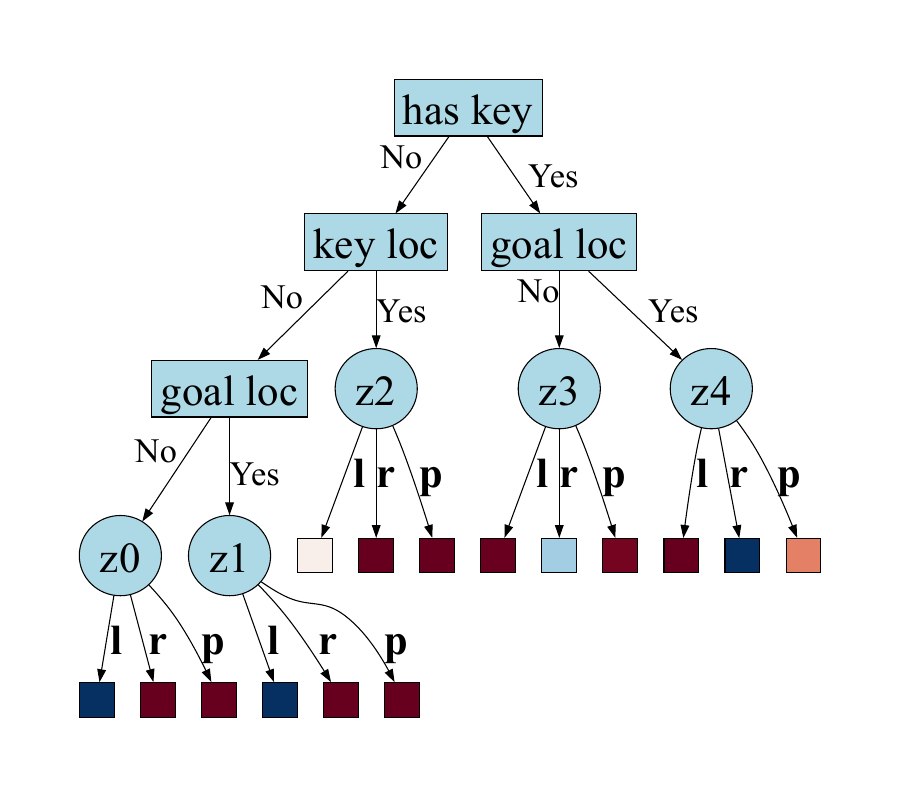}
    }
    \end{minipage}
    \begin{minipage}[b]{.111\textwidth}
    \centering
    {
        \includegraphics[height=3.2cm]{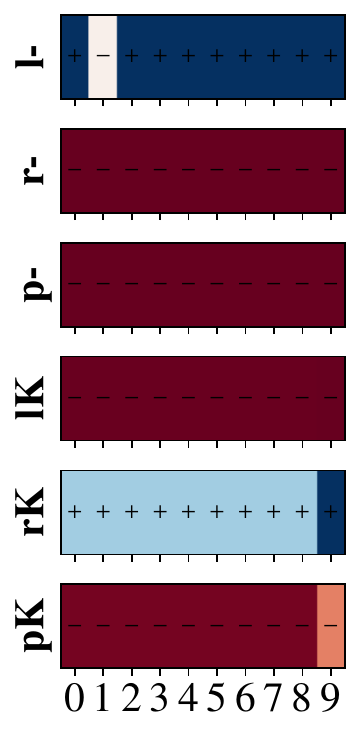}
    }
    \end{minipage}
    \begin{minipage}[b]{.215\textwidth}
    \centering
    {
        \includegraphics[trim={1.3cm 1.3cm 1.3cm 1.3cm}, clip, height=3.1cm]{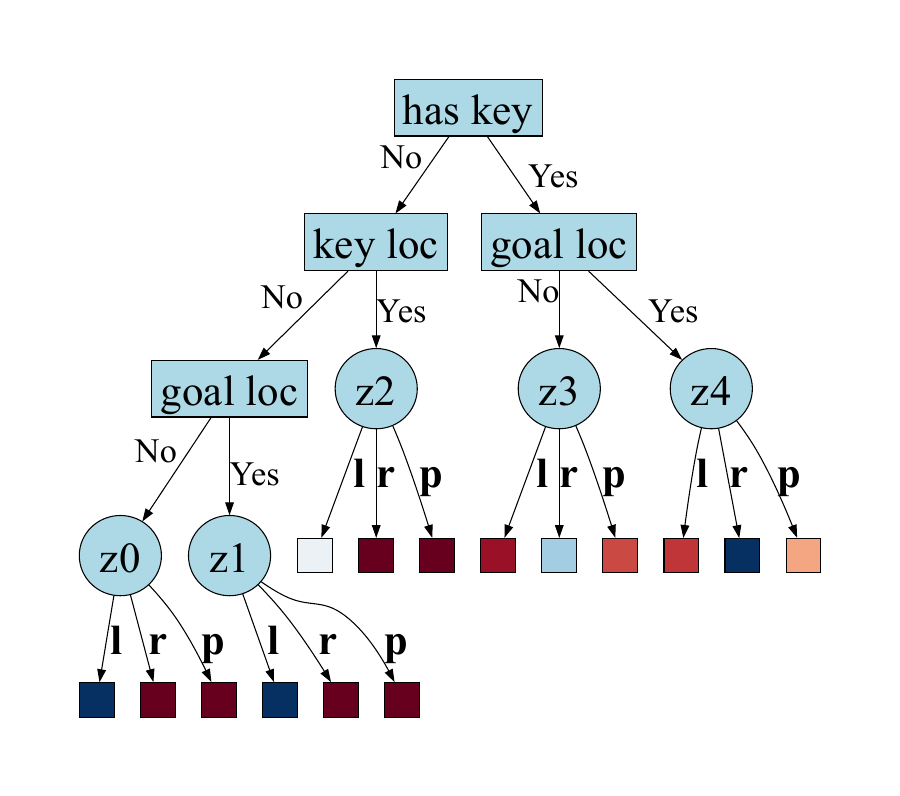}
    }
    \end{minipage}
    \begin{minipage}[b]{.111\textwidth}
    \centering
    {
        \includegraphics[height=3.2cm]{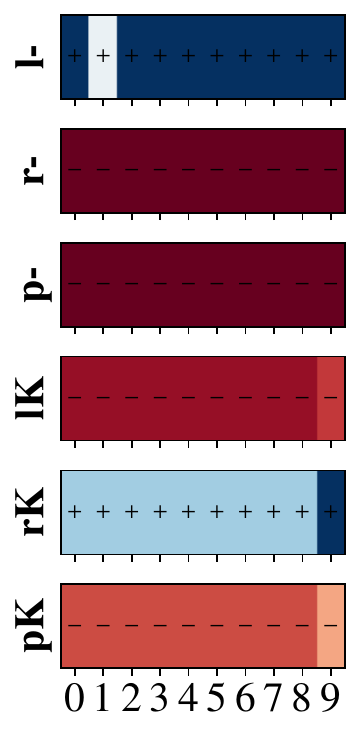}
    }
    \end{minipage}
    }
    \vspace{-4mm}     
    \caption{\envlinekey{}: ${R}^{\AlgOurs}$ for learner $2$  at $k=100, 30000, \text{and } 50000$ episodes.}
    \label{fig:designed.rewards.linekey.pool.Ada_TL.L2}
\end{subfigure}
\vspace{2mm}
\begin{subfigure}[b]{\textwidth}
    \centering{
    \begin{minipage}[b]{.215\textwidth}
    \centering
    { 
        \includegraphics[trim={1.3cm 1.3cm 1.3cm 1.3cm}, clip, height=3.1cm]{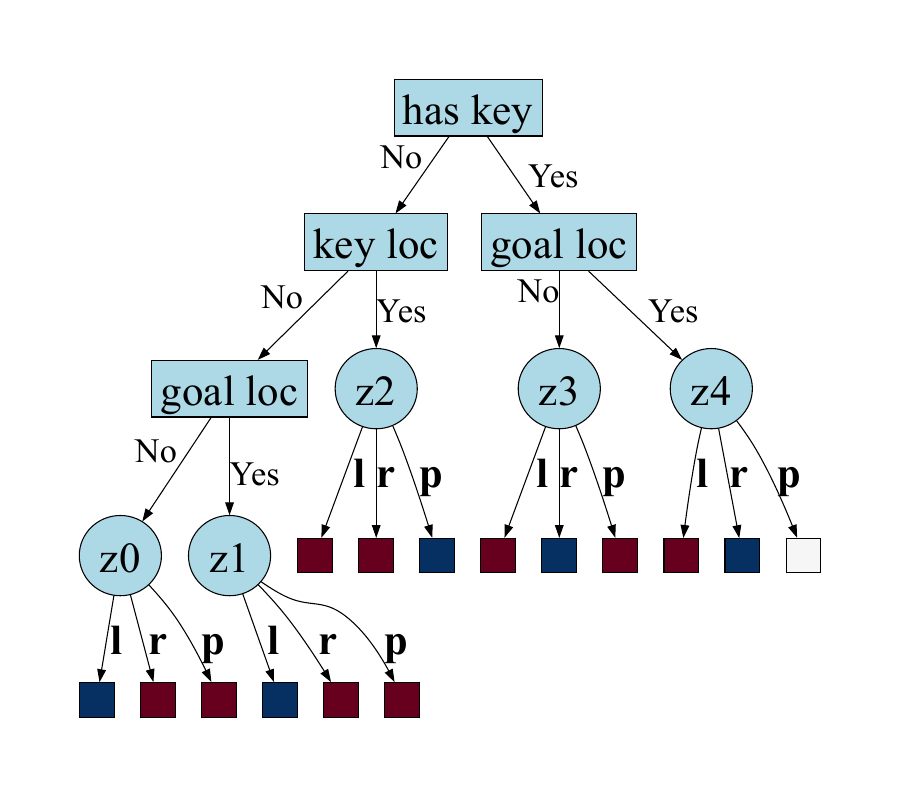}
    }
    \end{minipage}
    \begin{minipage}[b]{.111\textwidth}
    \centering
    {
        \includegraphics[height=3.2cm]{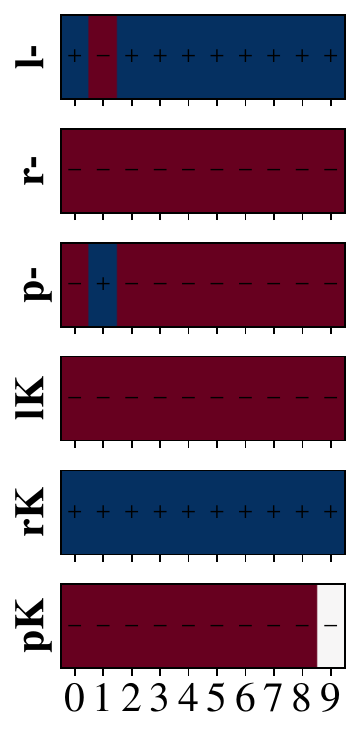}
    }
    \end{minipage}
    \begin{minipage}[b]{.215\textwidth}
    \centering
    {
        \includegraphics[trim={1.3cm 1.3cm 1.3cm 1.3cm}, clip, height=3.1cm]{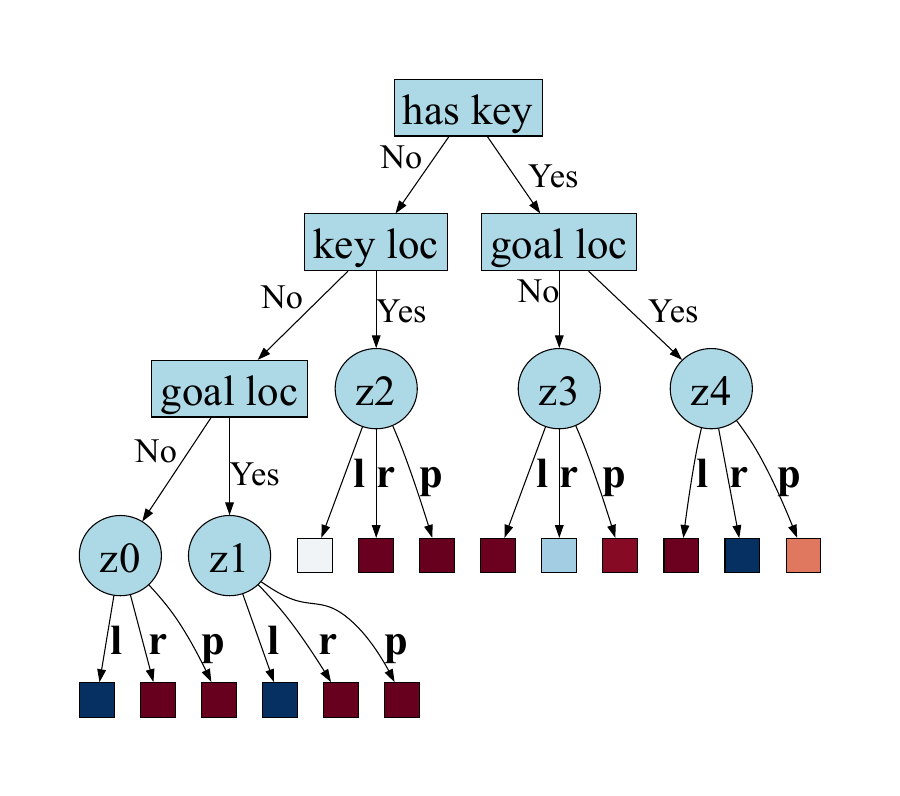}
    }
    \end{minipage}
    \begin{minipage}[b]{.111\textwidth}
    \centering
    {
        \includegraphics[height=3.2cm]{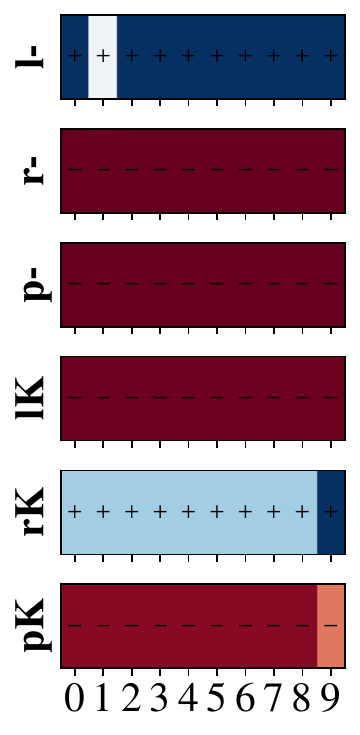}
    }
    \end{minipage}
    \begin{minipage}[b]{.215\textwidth}
    \centering
    {
        \includegraphics[trim={1.3cm 1.3cm 1.3cm 1.3cm}, clip, height=3.1cm]{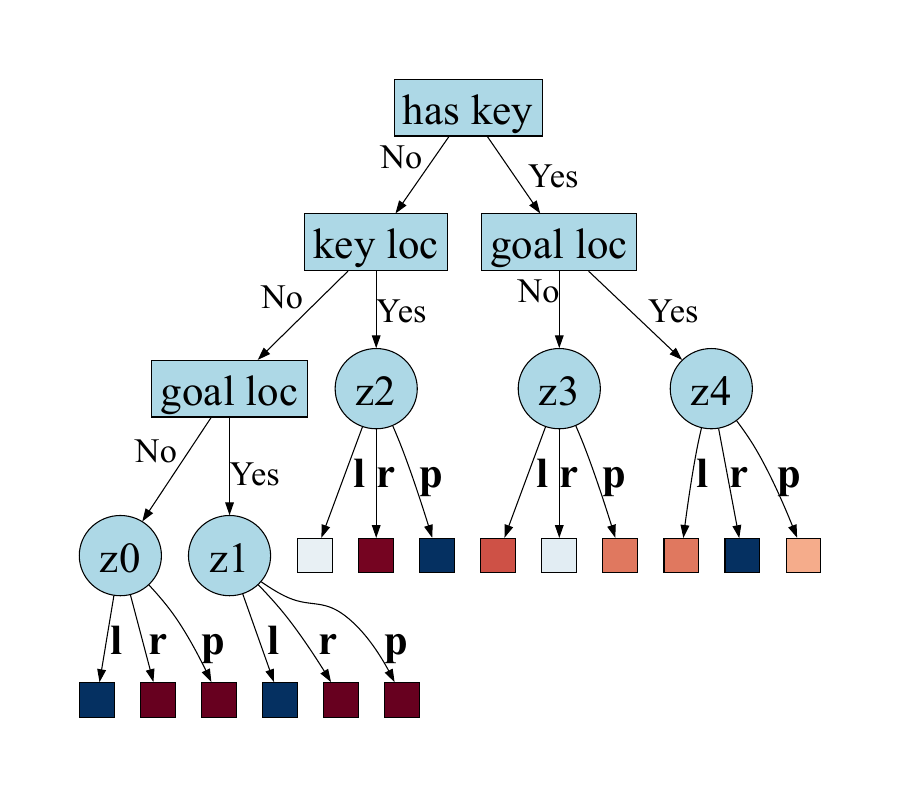}
    }
    \end{minipage}
    \begin{minipage}[b]{.111\textwidth}
    \centering
    {
        \includegraphics[height=3.2cm]{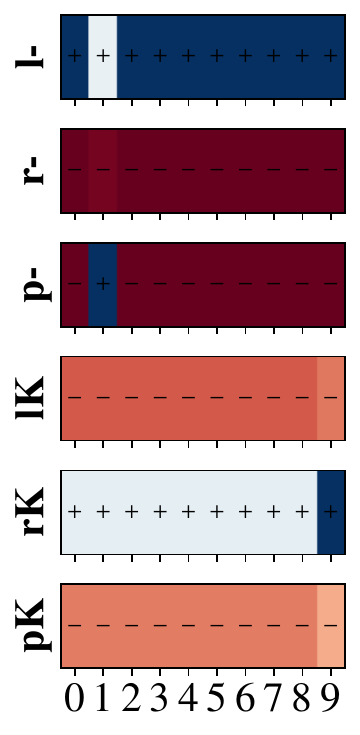}
    }
    \end{minipage}
    }
    \vspace{-4mm}     
    \caption{\envlinekey{}: ${R}^{\AlgOurs}$ for learner $3$  at $k=100, 30000, \text{and } 50000$ episodes.}
    \label{fig:designed.rewards.linekey.pool.Ada_TL.L3}
\end{subfigure}
    \vspace{-7.5mm}
    \caption{\looseness-1Visualization of reward functions designed by different techniques in the \envlinekey{} environment for all three actions $\{\textnormal{``left''}, \textnormal{``right''}, \textnormal{``pick''}\}$. \textbf{(a)} shows original reward function ${R}^{\orig}$. \textbf{(b)} shows reward function ${R}^{\Invariance}$. \textbf{(c)} shows reward function ${R}^{\EXPRD}$ designed by expert-driven non-adaptive reward design technique~\cite{devidze2021explicable}. \textbf{(d, e, f)} show reward functions ${R}^{\AlgOurs}$ designed by our framework \AlgOurs{} for three learners, each with its distinct initial policy, at different training episodes $k$. These plots illustrate reward values for all combinations of triplets: agent's location (indicated as ``key loc'', ``goal loc'' in tree plots),  agent's status whether it has acquired the key or not (indicated as ``has key'' in tree plots and letter ``K'' in bar plots), and three actions (indicated as `l' for ``left'', `r' for ``right'', `p' for ``pick''). A negative reward is shown in Red color with the sign ``-'', a positive reward is shown in Blue color with the sign ``+'', and a zero reward is shown in white. The color intensity indicates the reward magnitude.}
    \label{fig-app:designed.rewards.linekey}
    \Description{Designed rewards for the line-key environment.}
    \vspace{-2.5mm}
\end{figure*}

\looseness-1\textbf{Results.} Figure~\ref{fig:convergence.results.room} presents the results for both settings (i.e., a single learner and a diverse group of learners). The reported results are averaged over $40$ runs (where each run corresponds to designing rewards for a specific learner), and convergence plots show the mean performance with standard error bars.\footnote{We conducted the experiments on a cluster consisting of machines equipped with a 3.30 GHz Intel Xeon CPU E5-2667 v2 processor and 256 GB of RAM.} As evident from the results in Figures~\ref{fig:convergence.results.room.reinforce1}~and~\ref{fig:convergence.results.room.reinforce2}, the rewards designed by \AlgOurs{} significantly speed up the learner's convergence to optimal behavior when compared to the rewards designed by  baseline techniques. Notably, the effectiveness of \AlgOurs{} becomes more pronounced in scenarios featuring a diverse group of learners with distinct initial policies, where adaptive reward design plays a crucial role.
Figure~\ref{fig:designed.rewards.room} presents a visualization of the designed reward functions generated by different techniques at various episodes. Notably, the rewards ${R}^{\orig}$, ${R}^{\Invariance}$, and ${R}^{\EXPRD}$ are agnostic to the learner's policy and remain constant throughout the training process. In Figures~\ref{fig:designed.rewards.room.Ada_TL.L1},~\ref{fig:designed.rewards.room.Ada_TL.L2},~and~\ref{fig:designed.rewards.room.Ada_TL.L3}, we illustrate the ${R}_{k}^{\AlgOurs}$ rewards designed by our technique for three learners each with its distinct initial policy at $k=1000, 2000, 3000, 100000, \text{and } 200000$ episodes. As observed in these plots, \AlgOurs{} rapidly assigns high-magnitude numerical values to the designed rewards and adapts these rewards w.r.t. the learner's current policy. Initially (see $k=1000$ episode plots), the rewards designed by \AlgOurs{} encourage the agent to quickly reach the goal state (``green-star'') by providing positive reward signals for optimal actions (``up'', ``right'') followed by modifying reward signals in each episode to align with the learner's current policy.

\subsection{Evaluation on \envlinekey{}}
\label{sec:evaluation:envlinekey}

\looseness-1\textbf{\envlinekey{} (Figure~\ref{fig:envlinekey}).} 
This environment corresponds to a navigation task in a one-dimensional space where the agent has to first pick the key and then reach the goal. 
The environment used in our experiments is based on the work of \cite{devidze2021explicable} that also serves as a baseline technique. We represent the environment as an MDP with $\mathcal{S}$ states corresponding to nodes in a chain with the ``gray circle'' indicating the agent’s initial location. Goal (``green-star'') is available in the rightmost state, and the key is available at the state shown as
“cyan-bolt”. The agent can take three actions given by $\mathcal{A} := \{\textnormal{``left''}, \textnormal{``right''}, \textnormal{``pick''}\}$. \textnormal{``pick''} action does not change the agent's location, however, when executed in locations with the availability of the key, the agent acquires the key; if the agent already had a key, the action does not affect the status. A move action of \textnormal{``left''} or \textnormal{``right''} takes the agent from the current location to the neighboring node according to the direction of the action.
Similar to \envfourroom{}, the agent's move action is not applied if the new location crosses the wall, and there is $p_{\textnormal{rand}}$ probability of a random action. The agent gets a reward of $R_{\textnormal{max}}$ after it has navigated to the goal locations after acquiring the key and then takes a ``right'' action; note that this action also terminates the episode. The reward is $0$ elsewhere and there is a discount factor $\gamma$. We set $p_{\textnormal{rand}}=0.1$, $R_{\textnormal{max}}=10$, $\gamma=0.95$, and the environment resets after a horizon of $H=30$ steps.

\looseness-1\textbf{Reward structure.} We adopt a tree structured representation of the state space, as visually depicted in Figure~\ref{fig:envlinekey.tree.structure}. To formalize this representation, we employ a state abstraction function denoted as $\psi: \mathcal{S} \to \{0,1\}^{5}$. For each state $s \in \mathcal{S}$, the $i$-th entry of $\psi(s)$ is set to $1$ if $s$ maps to the $i$-th circled node of the tree (i.e., parent to leaf nodes), and $0$ otherwise. Then, we define the set $\mathcal{R}_\textnormal{str}$ in a manner similar to that outlined in Section~\ref{sec:evaluation:envfourroom}. Further, we define $\mathcal{R} := \mathcal{R}_\textnormal{inv} \cap \mathcal{R}_\textnormal{str}$ as discussed in Section~\ref{sec:expadard-formal}. We note that $\overline{R} \in \mathcal{R}$.

\looseness-1\textbf{Evaluation setup and techniques evaluated.} Our evaluation setup for \envlinekey{} environment is exactly the same as that used for \envfourroom{} environment (described in Section~\ref{sec:evaluation:envfourroom}). In particular, all the hyperparameters (related to the REINFORCE agent, reward design techniques, and training process) are the same as in Section~\ref{sec:evaluation:envfourroom}. 
In this evaluation, we again have two settings to evaluate the utility of adaptive reward design: (i) a single learner with a uniformly random initial policy (where each action is taken with a probability of $0.33$) and (ii) a diverse group of learners, each with distinct initial policies. To generate a collection of distinctive initial policies, we introduced modifications to a uniformly random policy. These modifications were designed to incorporate a $0.7$ probability of the agent selecting suboptimal actions from various states. In our evaluation, we included five such unique initial policies.

\looseness-1\textbf{Results.} Figure~\ref{fig:convergence.results.linekey} presents the results for both settings (i.e., a single learner and a diverse group of learners). The reported results are averaged over $30$ runs, and convergence plots show the mean performance with standard error bars. These results further demonstrate the effectiveness and robustness of \AlgOurs{} across different settings in comparison to baselines.
\looseness-1Analogous to Figure~\ref{fig:designed.rewards.room} in Section \ref{sec:evaluation:envfourroom}, Figure~\ref{fig-app:designed.rewards.linekey} presents a visualization of the designed reward functions produced by different techniques at various training episodes. These results illustrate the utility of our proposed informativeness criterion for adaptive reward design, particularly when dealing with various structural constraints to obtain interpretable rewards, including tree-structured reward functions.

\section{Concluding Discussions}
\label{sec:conclusions}
\looseness-1We studied the problem of expert-driven reward design, where an expert/teacher seeks to provide informative and interpretable rewards to a learning agent. We introduced a novel reward informativeness criterion that adapts w.r.t. the agent's current policy. Based on this informativeness criterion, we developed an expert-driven adaptive reward design framework, \AlgOurs. We empirically demonstrated the utility of our framework on two navigation tasks.

\looseness-1Next, we discuss a few limitations of our work and outline a future plan to address them. First, we conducted experiments on simpler environments to systematically investigate the effectiveness of our informativeness criterion in terms of adaptivity and structure of designed reward functions. It would be interesting to extend the evaluation of the reward design framework in more complex environments (e.g., with continuous state/action spaces) by leveraging an abstraction-based pipeline considered in~\cite{devidze2021explicable}. Second, we considered fixed structural properties to induce interpretable reward functions. It would also be interesting to investigate the usage of our informativeness criterion for automatically discovering or optimizing the structured properties (e.g., nodes in the tree structure).
Third, we empirically showed the effectiveness of our adaptive rewards, but adaptive rewards could also lead to instability in the agent's learning process. It would be useful to analyze our adaptive reward design framework in terms of an agent's convergence speed and stability.



\begin{acks}
Funded/Co-funded by the European Union (ERC, TOPS, 101039090). Views and opinions expressed are however those of the author(s) only and do not necessarily reflect those of the European Union or the European Research Council. Neither the European Union nor the granting authority can be held responsible for them.
\end{acks}


\vspace{-1.5mm}
\section*{Ethics Statement} 
This work presents a reward informativeness criterion that can be utilized in designing adaptive, informative, and interpretable rewards for a learning agent. Given the algorithmic nature of our work applied to agents, we do not foresee direct negative societal impacts of our work in the present form.

\clearpage

\balance

 
\bibliographystyle{ACM-Reference-Format} 
\bibliography{main}


\begin{thebibliography}{48}


\ifx \showCODEN    \undefined \def \showCODEN     #1{\unskip}     \fi
\ifx \showDOI      \undefined \def \showDOI       #1{#1}\fi
\ifx \showISBNx    \undefined \def \showISBNx     #1{\unskip}     \fi
\ifx \showISBNxiii \undefined \def \showISBNxiii  #1{\unskip}     \fi
\ifx \showISSN     \undefined \def \showISSN      #1{\unskip}     \fi
\ifx \showLCCN     \undefined \def \showLCCN      #1{\unskip}     \fi
\ifx \shownote     \undefined \def \shownote      #1{#1}          \fi
\ifx \showarticletitle \undefined \def \showarticletitle #1{#1}   \fi
\ifx \showURL      \undefined \def \showURL       {\relax}        \fi
\providecommand\bibfield[2]{#2}
\providecommand\bibinfo[2]{#2}
\providecommand\natexlab[1]{#1}
\providecommand\showeprint[2][]{arXiv:#2}

\bibitem[\protect\citeauthoryear{Andrychowicz, Denil, Colmenarejo, Hoffman,
  Pfau, Schaul, and de~Freitas}{Andrychowicz et~al\mbox{.}}{2016}]%
        {andrychowicz2016learning}
\bibfield{author}{\bibinfo{person}{Marcin Andrychowicz}, \bibinfo{person}{Misha
  Denil}, \bibinfo{person}{Sergio~Gomez Colmenarejo},
  \bibinfo{person}{Matthew~W. Hoffman}, \bibinfo{person}{David Pfau},
  \bibinfo{person}{Tom Schaul}, {and} \bibinfo{person}{Nando de Freitas}.}
  \bibinfo{year}{2016}\natexlab{}.
\newblock \showarticletitle{Learning to {L}earn by {G}radient {D}escent by
  {G}radient {D}escent}. In \bibinfo{booktitle}{\emph{NeurIPS}}.
  \bibinfo{pages}{3981--3989}.
\newblock


\bibitem[\protect\citeauthoryear{Arjona{-}Medina, Gillhofer, Widrich,
  Unterthiner, Brandstetter, and Hochreiter}{Arjona{-}Medina
  et~al\mbox{.}}{2019}]%
        {DBLP:conf/nips/Arjona-MedinaGW19}
\bibfield{author}{\bibinfo{person}{Jose~A. Arjona{-}Medina},
  \bibinfo{person}{Michael Gillhofer}, \bibinfo{person}{Michael Widrich},
  \bibinfo{person}{Thomas Unterthiner}, \bibinfo{person}{Johannes
  Brandstetter}, {and} \bibinfo{person}{Sepp Hochreiter}.}
  \bibinfo{year}{2019}\natexlab{}.
\newblock \showarticletitle{{RUDDER:} {R}eturn {D}ecomposition for {D}elayed
  {R}ewards}. In \bibinfo{booktitle}{\emph{NeurIPS}}.
  \bibinfo{pages}{13544--13555}.
\newblock


\bibitem[\protect\citeauthoryear{Asmuth, Littman, and Zinkov}{Asmuth
  et~al\mbox{.}}{2008}]%
        {DBLP:conf/aaai/AsmuthLZ08}
\bibfield{author}{\bibinfo{person}{John Asmuth}, \bibinfo{person}{Michael~L.
  Littman}, {and} \bibinfo{person}{Robert Zinkov}.}
  \bibinfo{year}{2008}\natexlab{}.
\newblock \showarticletitle{Potential-{b}ased {S}haping in {M}odel-{b}ased
  {R}einforcement {L}earning}. In \bibinfo{booktitle}{\emph{AAAI}}.
  \bibinfo{publisher}{{AAAI} Press}, \bibinfo{pages}{604--609}.
\newblock


\bibitem[\protect\citeauthoryear{Barto}{Barto}{2013}]%
        {DBLP:books/sp/13/Barto13}
\bibfield{author}{\bibinfo{person}{Andrew~G. Barto}.}
  \bibinfo{year}{2013}\natexlab{}.
\newblock \showarticletitle{Intrinsic {M}otivation and {R}einforcement
  {L}earning}.
\newblock In \bibinfo{booktitle}{\emph{Intrinsically Motivated Learning in
  Natural and Artificial Systems}}. \bibinfo{publisher}{Springer},
  \bibinfo{pages}{17--47}.
\newblock


\bibitem[\protect\citeauthoryear{Bewley and L{\'{e}}cu{\'{e}}}{Bewley and
  L{\'{e}}cu{\'{e}}}{2022}]%
        {BewleyL22}
\bibfield{author}{\bibinfo{person}{Tom Bewley} {and} \bibinfo{person}{Freddy
  L{\'{e}}cu{\'{e}}}.} \bibinfo{year}{2022}\natexlab{}.
\newblock \showarticletitle{{I}nterpretable {P}reference-based {R}einforcement
  {L}earning with {T}ree-{S}tructured {R}eward {F}unctions}. In
  \bibinfo{booktitle}{\emph{{AAMAS}}}. \bibinfo{publisher}{International
  Foundation for Autonomous Agents and Multiagent Systems},
  \bibinfo{pages}{118--126}.
\newblock


\bibitem[\protect\citeauthoryear{Brys, Harutyunyan, Suay, Chernova, Taylor, and
  Now{\'e}}{Brys et~al\mbox{.}}{2015}]%
        {brys2015reinforcement}
\bibfield{author}{\bibinfo{person}{Tim Brys}, \bibinfo{person}{Anna
  Harutyunyan}, \bibinfo{person}{Halit~Bener Suay}, \bibinfo{person}{Sonia
  Chernova}, \bibinfo{person}{Matthew~E Taylor}, {and} \bibinfo{person}{Ann
  Now{\'e}}.} \bibinfo{year}{2015}\natexlab{}.
\newblock \showarticletitle{{R}einforcement {L}earning from {D}emonstration
  through {S}haping}. In \bibinfo{booktitle}{\emph{IJCAI}}.
  \bibinfo{publisher}{{AAAI} Press}, \bibinfo{pages}{3352--3358}.
\newblock


\bibitem[\protect\citeauthoryear{Camacho, Chen, Sanner, and McIlraith}{Camacho
  et~al\mbox{.}}{2017}]%
        {camacho2017decision}
\bibfield{author}{\bibinfo{person}{Alberto Camacho}, \bibinfo{person}{Oscar
  Chen}, \bibinfo{person}{Scott Sanner}, {and} \bibinfo{person}{Sheila~A
  McIlraith}.} \bibinfo{year}{2017}\natexlab{}.
\newblock \showarticletitle{Decision-{M}aking with {N}on-{M}arkovian {R}ewards:
  {F}rom {LTL} to {A}utomata-based {R}eward {S}haping}. In
  \bibinfo{booktitle}{\emph{RLDM}}. \bibinfo{pages}{279--283}.
\newblock


\bibitem[\protect\citeauthoryear{Dai and Walter}{Dai and Walter}{2019}]%
        {DBLP:conf/nips/DaiW19}
\bibfield{author}{\bibinfo{person}{Falcon~Z. Dai} {and}
  \bibinfo{person}{Matthew~R. Walter}.} \bibinfo{year}{2019}\natexlab{}.
\newblock \showarticletitle{Maximum {E}xpected {H}itting {C}ost of a {M}arkov
  {D}ecision {P}rocess and {I}nformativeness of {R}ewards}. In
  \bibinfo{booktitle}{\emph{NeurIPS}}. \bibinfo{pages}{7677--7685}.
\newblock


\bibitem[\protect\citeauthoryear{Daniel, Viering, Metz, Kroemer, and
  Peters}{Daniel et~al\mbox{.}}{2014}]%
        {daniel2014active}
\bibfield{author}{\bibinfo{person}{Christian Daniel}, \bibinfo{person}{Malte
  Viering}, \bibinfo{person}{Jan Metz}, \bibinfo{person}{Oliver Kroemer}, {and}
  \bibinfo{person}{Jan Peters}.} \bibinfo{year}{2014}\natexlab{}.
\newblock \showarticletitle{{A}ctive {R}eward {L}earning.}. In
  \bibinfo{booktitle}{\emph{Robotics: Science and Systems}}.
\newblock


\bibitem[\protect\citeauthoryear{De~Giacomo, Favorito, Iocchi, and
  Patrizi}{De~Giacomo et~al\mbox{.}}{2020}]%
        {de2020imitation}
\bibfield{author}{\bibinfo{person}{Giuseppe De~Giacomo}, \bibinfo{person}{Marco
  Favorito}, \bibinfo{person}{Luca Iocchi}, {and} \bibinfo{person}{Fabio
  Patrizi}.} \bibinfo{year}{2020}\natexlab{}.
\newblock \showarticletitle{{I}mitation {L}earning over {H}eterogeneous
  {A}gents with {R}estraining {B}olts}. In \bibinfo{booktitle}{\emph{ICAPS}},
  Vol.~\bibinfo{volume}{30}. \bibinfo{publisher}{{AAAI} Press},
  \bibinfo{pages}{517--521}.
\newblock


\bibitem[\protect\citeauthoryear{Demir, {\c{C}}ilden, and Polat}{Demir
  et~al\mbox{.}}{2019}]%
        {DBLP:conf/atal/0002CP19}
\bibfield{author}{\bibinfo{person}{Alper Demir}, \bibinfo{person}{Erkin
  {\c{C}}ilden}, {and} \bibinfo{person}{Faruk Polat}.}
  \bibinfo{year}{2019}\natexlab{}.
\newblock \showarticletitle{Landmark {B}ased {R}eward {S}haping in
  {R}einforcement {L}earning with {H}idden {S}tates}. In
  \bibinfo{booktitle}{\emph{AAMAS}}. \bibinfo{publisher}{International
  Foundation for Autonomous Agents and Multiagent Systems},
  \bibinfo{pages}{1922--1924}.
\newblock


\bibitem[\protect\citeauthoryear{Devidze, Kamalaruban, and Singla}{Devidze
  et~al\mbox{.}}{2022}]%
        {devidze2022exploration}
\bibfield{author}{\bibinfo{person}{Rati Devidze}, \bibinfo{person}{Parameswaran
  Kamalaruban}, {and} \bibinfo{person}{Adish Singla}.}
  \bibinfo{year}{2022}\natexlab{}.
\newblock \showarticletitle{Exploration-{G}uided {R}eward {S}haping for
  {R}einforcement {L}earning under {S}parse {R}ewards}. In
  \bibinfo{booktitle}{\emph{NeurIPS}}. \bibinfo{pages}{5829--5842}.
\newblock


\bibitem[\protect\citeauthoryear{Devidze, Radanovic, Kamalaruban, and
  Singla}{Devidze et~al\mbox{.}}{2021}]%
        {devidze2021explicable}
\bibfield{author}{\bibinfo{person}{Rati Devidze}, \bibinfo{person}{Goran
  Radanovic}, \bibinfo{person}{Parameswaran Kamalaruban}, {and}
  \bibinfo{person}{Adish Singla}.} \bibinfo{year}{2021}\natexlab{}.
\newblock \showarticletitle{Explicable {R}eward {D}esign for {R}einforcement
  {L}earning {A}gents}. In \bibinfo{booktitle}{\emph{NeurIPS}}.
  \bibinfo{pages}{20118--20131}.
\newblock


\bibitem[\protect\citeauthoryear{Devlin and Kudenko}{Devlin and
  Kudenko}{2012}]%
        {DBLP:conf/aamas/DevlinK12}
\bibfield{author}{\bibinfo{person}{Sam Devlin} {and} \bibinfo{person}{Daniel
  Kudenko}.} \bibinfo{year}{2012}\natexlab{}.
\newblock \showarticletitle{Dynamic {P}otential-based {R}eward {S}haping}. In
  \bibinfo{booktitle}{\emph{AAMAS}}. \bibinfo{publisher}{International
  Foundation for Autonomous Agents and Multiagent Systems},
  \bibinfo{pages}{433--440}.
\newblock


\bibitem[\protect\citeauthoryear{Ferret, Marinier, Geist, and Pietquin}{Ferret
  et~al\mbox{.}}{2020}]%
        {DBLP:conf/ijcai/FerretMGP20}
\bibfield{author}{\bibinfo{person}{Johan Ferret},
  \bibinfo{person}{Rapha{\"{e}}l Marinier}, \bibinfo{person}{Matthieu Geist},
  {and} \bibinfo{person}{Olivier Pietquin}.} \bibinfo{year}{2020}\natexlab{}.
\newblock \showarticletitle{Self-{A}ttentional {C}redit {A}ssignment for
  {T}ransfer in {R}einforcement {L}earning}. In
  \bibinfo{booktitle}{\emph{IJCAI}}. \bibinfo{publisher}{ijcai.org},
  \bibinfo{pages}{2655--2661}.
\newblock


\bibitem[\protect\citeauthoryear{Goyal, Niekum, and Mooney}{Goyal
  et~al\mbox{.}}{2019}]%
        {DBLP:conf/ijcai/GoyalNM19}
\bibfield{author}{\bibinfo{person}{Prasoon Goyal}, \bibinfo{person}{Scott
  Niekum}, {and} \bibinfo{person}{Raymond~J. Mooney}.}
  \bibinfo{year}{2019}\natexlab{}.
\newblock \showarticletitle{Using {N}atural {L}anguage for {R}eward {S}haping
  in {R}einforcement {L}earning}. In \bibinfo{booktitle}{\emph{IJCAI}}.
  \bibinfo{publisher}{ijcai.org}, \bibinfo{pages}{2385--2391}.
\newblock


\bibitem[\protect\citeauthoryear{Grzes}{Grzes}{2017}]%
        {DBLP:conf/atal/Grzes17}
\bibfield{author}{\bibinfo{person}{Marek Grzes}.}
  \bibinfo{year}{2017}\natexlab{}.
\newblock \showarticletitle{Reward {S}haping in {E}pisodic {R}einforcement
  {L}earning}. In \bibinfo{booktitle}{\emph{AAMAS}}.
  \bibinfo{publisher}{{ACM}}, \bibinfo{pages}{565--573}.
\newblock


\bibitem[\protect\citeauthoryear{Grzes and Kudenko}{Grzes and Kudenko}{2008}]%
        {grzes2008plan}
\bibfield{author}{\bibinfo{person}{Marek Grzes} {and} \bibinfo{person}{Daniel
  Kudenko}.} \bibinfo{year}{2008}\natexlab{}.
\newblock \showarticletitle{Plan-based {R}eward {S}haping for {R}einforcement
  {L}earning}. In \bibinfo{booktitle}{\emph{International Conference on
  Intelligent Systems}}, Vol.~\bibinfo{volume}{2}. IEEE,
  \bibinfo{pages}{10--22}.
\newblock


\bibitem[\protect\citeauthoryear{Icarte, Klassen, Valenzano, and
  McIlraith}{Icarte et~al\mbox{.}}{2022}]%
        {DBLP:journals/jair/IcarteKVM22}
\bibfield{author}{\bibinfo{person}{Rodrigo~Toro Icarte},
  \bibinfo{person}{Toryn~Q. Klassen}, \bibinfo{person}{Richard~Anthony
  Valenzano}, {and} \bibinfo{person}{Sheila~A. McIlraith}.}
  \bibinfo{year}{2022}\natexlab{}.
\newblock \showarticletitle{{R}eward {M}achines: {E}xploiting {R}eward
  {F}unction {S}tructure in {R}einforcement {L}earning}.
\newblock \bibinfo{journal}{\emph{Journal of Artificial Intelligence Research}}
   \bibinfo{volume}{73} (\bibinfo{year}{2022}), \bibinfo{pages}{173--208}.
\newblock


\bibitem[\protect\citeauthoryear{James and Singh}{James and Singh}{2009}]%
        {DBLP:conf/atal/JamesS09}
\bibfield{author}{\bibinfo{person}{Michael~R. James} {and}
  \bibinfo{person}{Satinder~P. Singh}.} \bibinfo{year}{2009}\natexlab{}.
\newblock \showarticletitle{Sarsa{L}andmark: {A}n {A}lgorithm for {L}earning in
  {POMDP}s with {L}andmarks}. In \bibinfo{booktitle}{\emph{AAMAS}}.
  \bibinfo{publisher}{International Foundation for Autonomous Agents and
  Multiagent Systems}, \bibinfo{pages}{585--591}.
\newblock


\bibitem[\protect\citeauthoryear{Jiang, Bharadwaj, Wu, Shah, Topcu, and
  Stone}{Jiang et~al\mbox{.}}{2021}]%
        {jiang2021temporal_AAAI}
\bibfield{author}{\bibinfo{person}{Yuqian Jiang}, \bibinfo{person}{Suda
  Bharadwaj}, \bibinfo{person}{Bo Wu}, \bibinfo{person}{Rishi Shah},
  \bibinfo{person}{Ufuk Topcu}, {and} \bibinfo{person}{Peter Stone}.}
  \bibinfo{year}{2021}\natexlab{}.
\newblock \showarticletitle{Temporal-{L}ogic-{B}ased {R}eward {S}haping for
  {C}ontinuing {R}einforcement {L}earning {T}asks}. In
  \bibinfo{booktitle}{\emph{AAAI}}. \bibinfo{publisher}{{AAAI} Press},
  \bibinfo{pages}{7995--8003}.
\newblock


\bibitem[\protect\citeauthoryear{Jothimurugan, Alur, and Bastani}{Jothimurugan
  et~al\mbox{.}}{2019}]%
        {DBLP:conf/nips/JothimuruganAB19}
\bibfield{author}{\bibinfo{person}{Kishor Jothimurugan},
  \bibinfo{person}{Rajeev Alur}, {and} \bibinfo{person}{Osbert Bastani}.}
  \bibinfo{year}{2019}\natexlab{}.
\newblock \showarticletitle{A {C}omposable {S}pecification {L}anguage for
  {R}einforcement {L}earning {T}asks}. In \bibinfo{booktitle}{\emph{NeurIPS}}.
  \bibinfo{pages}{13021--13030}.
\newblock


\bibitem[\protect\citeauthoryear{Kamalaruban, Devidze, Cevher, and
  Singla}{Kamalaruban et~al\mbox{.}}{2020}]%
        {DBLP:journals/corr/abs-2006-13160}
\bibfield{author}{\bibinfo{person}{Parameswaran Kamalaruban},
  \bibinfo{person}{Rati Devidze}, \bibinfo{person}{Volkan Cevher}, {and}
  \bibinfo{person}{Adish Singla}.} \bibinfo{year}{2020}\natexlab{}.
\newblock \showarticletitle{Environment {S}haping in {R}einforcement {L}earning
  using {S}tate {A}bstraction}.
\newblock \bibinfo{journal}{\emph{CoRR}}  \bibinfo{volume}{abs/2006.13160}
  (\bibinfo{year}{2020}).
\newblock


\bibitem[\protect\citeauthoryear{Kulkarni, Narasimhan, Saeedi, and
  Tenenbaum}{Kulkarni et~al\mbox{.}}{2016}]%
        {DBLP:conf/nips/KulkarniNST16}
\bibfield{author}{\bibinfo{person}{Tejas~D. Kulkarni}, \bibinfo{person}{Karthik
  Narasimhan}, \bibinfo{person}{Ardavan Saeedi}, {and} \bibinfo{person}{Josh
  Tenenbaum}.} \bibinfo{year}{2016}\natexlab{}.
\newblock \showarticletitle{Hierarchical {D}eep {R}einforcement {L}earning:
  {I}ntegrating {T}emporal {A}bstraction and {I}ntrinsic {M}otivation}. In
  \bibinfo{booktitle}{\emph{NeurIPS}}. \bibinfo{pages}{3675--3683}.
\newblock


\bibitem[\protect\citeauthoryear{Laud and DeJong}{Laud and DeJong}{2003}]%
        {DBLP:conf/icml/LaudD03}
\bibfield{author}{\bibinfo{person}{Adam Laud} {and} \bibinfo{person}{Gerald
  DeJong}.} \bibinfo{year}{2003}\natexlab{}.
\newblock \showarticletitle{The {I}nfluence of {R}eward on the {S}peed of
  {R}einforcement {L}earning: {A}n {A}nalysis of {S}haping}. In
  \bibinfo{booktitle}{\emph{ICML}}. \bibinfo{publisher}{{AAAI} Press},
  \bibinfo{pages}{440--447}.
\newblock


\bibitem[\protect\citeauthoryear{Ma, Zhang, Sun, and Zhu}{Ma
  et~al\mbox{.}}{2019}]%
        {DBLP:conf/nips/MaZSZ19}
\bibfield{author}{\bibinfo{person}{Yuzhe Ma}, \bibinfo{person}{Xuezhou Zhang},
  \bibinfo{person}{Wen Sun}, {and} \bibinfo{person}{Jerry Zhu}.}
  \bibinfo{year}{2019}\natexlab{}.
\newblock \showarticletitle{Policy {P}oisoning in {B}atch {R}einforcement
  {L}earning and {C}ontrol}. In \bibinfo{booktitle}{\emph{NeurIPS}}.
  \bibinfo{pages}{14543--14553}.
\newblock


\bibitem[\protect\citeauthoryear{Maloney, Peppler, Kafai, Resnick, and
  Rusk}{Maloney et~al\mbox{.}}{2008}]%
        {maloney2008programming}
\bibfield{author}{\bibinfo{person}{John~H. Maloney}, \bibinfo{person}{Kylie
  Peppler}, \bibinfo{person}{Yasmin Kafai}, \bibinfo{person}{Mitchel Resnick},
  {and} \bibinfo{person}{Natalie Rusk}.} \bibinfo{year}{2008}\natexlab{}.
\newblock \showarticletitle{{P}rogramming by {C}hoice: {Urban} {Y}outh
  {L}earning {P}rogramming with {Scratch}}. In
  \bibinfo{booktitle}{\emph{SIGCSE}}. \bibinfo{publisher}{{ACM}},
  \bibinfo{pages}{367--371}.
\newblock


\bibitem[\protect\citeauthoryear{Mataric}{Mataric}{1994}]%
        {DBLP:conf/icml/Mataric94}
\bibfield{author}{\bibinfo{person}{Maja~J. Mataric}.}
  \bibinfo{year}{1994}\natexlab{}.
\newblock \showarticletitle{Reward {F}unctions for {A}ccelerated {L}earning}.
  In \bibinfo{booktitle}{\emph{ICML}}. \bibinfo{publisher}{Morgan Kaufmann},
  \bibinfo{pages}{181--189}.
\newblock


\bibitem[\protect\citeauthoryear{McGovern and Barto}{McGovern and
  Barto}{2001}]%
        {DBLP:conf/icml/McGovernB01}
\bibfield{author}{\bibinfo{person}{Amy McGovern} {and}
  \bibinfo{person}{Andrew~G. Barto}.} \bibinfo{year}{2001}\natexlab{}.
\newblock \showarticletitle{Automatic {D}iscovery of {S}ubgoals in
  {R}einforcement {L}earning using {D}iverse {D}ensity}. In
  \bibinfo{booktitle}{\emph{ICML}}. \bibinfo{publisher}{Morgan Kaufmann},
  \bibinfo{pages}{361--368}.
\newblock


\bibitem[\protect\citeauthoryear{Memarian, Goo, Lioutikov, Niekum, and
  Topcu}{Memarian et~al\mbox{.}}{2021}]%
        {memarian2021self}
\bibfield{author}{\bibinfo{person}{Farzan Memarian}, \bibinfo{person}{Wonjoon
  Goo}, \bibinfo{person}{Rudolf Lioutikov}, \bibinfo{person}{Scott Niekum},
  {and} \bibinfo{person}{Ufuk Topcu}.} \bibinfo{year}{2021}\natexlab{}.
\newblock \showarticletitle{Self-{S}upervised {O}nline {R}eward {S}haping in
  {S}parse-{R}eward {E}nvironments}. In \bibinfo{booktitle}{\emph{IROS}}. IEEE,
  \bibinfo{pages}{2369--2375}.
\newblock


\bibitem[\protect\citeauthoryear{Ng, Harada, and Russell}{Ng
  et~al\mbox{.}}{1999}]%
        {DBLP:conf/icml/NgHR99}
\bibfield{author}{\bibinfo{person}{Andrew~Y. Ng}, \bibinfo{person}{Daishi
  Harada}, {and} \bibinfo{person}{Stuart~J. Russell}.}
  \bibinfo{year}{1999}\natexlab{}.
\newblock \showarticletitle{Policy {I}nvariance {U}nder {R}eward
  {T}ransformations: {T}heory and {A}pplication to {R}eward {S}haping}. In
  \bibinfo{booktitle}{\emph{ICML}}. \bibinfo{publisher}{Morgan Kaufmann},
  \bibinfo{pages}{278--287}.
\newblock


\bibitem[\protect\citeauthoryear{Nichol, Achiam, and Schulman}{Nichol
  et~al\mbox{.}}{2018}]%
        {nichol2018first}
\bibfield{author}{\bibinfo{person}{Alex Nichol}, \bibinfo{person}{Joshua
  Achiam}, {and} \bibinfo{person}{John Schulman}.}
  \bibinfo{year}{2018}\natexlab{}.
\newblock \showarticletitle{On {F}irst-{O}rder {M}eta-{L}earning {A}lgorithms}.
\newblock \bibinfo{journal}{\emph{CoRR}}  \bibinfo{volume}{abs/1803.02999}
  (\bibinfo{year}{2018}).
\newblock


\bibitem[\protect\citeauthoryear{O'Rourke, Haimovitz, Ballweber, Dweck, and
  Popovic}{O'Rourke et~al\mbox{.}}{2014}]%
        {DBLP:conf/chi/ORourkeHBDP14}
\bibfield{author}{\bibinfo{person}{Eleanor O'Rourke}, \bibinfo{person}{Kyla
  Haimovitz}, \bibinfo{person}{Christy Ballweber}, \bibinfo{person}{Carol~S.
  Dweck}, {and} \bibinfo{person}{Zoran Popovic}.}
  \bibinfo{year}{2014}\natexlab{}.
\newblock \showarticletitle{Brain {P}oints: {A} {G}rowth {M}indset {I}ncentive
  {S}tructure {B}oosts {P}ersistence in an {E}ducational {G}ame}. In
  \bibinfo{booktitle}{\emph{{CHI}}}. \bibinfo{publisher}{{ACM}},
  \bibinfo{pages}{3339--3348}.
\newblock


\bibitem[\protect\citeauthoryear{Raileanu, Denton, Szlam, and Fergus}{Raileanu
  et~al\mbox{.}}{2018}]%
        {DBLP:conf/icml/RaileanuDSF18}
\bibfield{author}{\bibinfo{person}{Roberta Raileanu}, \bibinfo{person}{Emily
  Denton}, \bibinfo{person}{Arthur Szlam}, {and} \bibinfo{person}{Rob Fergus}.}
  \bibinfo{year}{2018}\natexlab{}.
\newblock \showarticletitle{Modeling {O}thers using {O}neself in
  {M}ulti-{A}gent {R}einforcement {L}earning}. In
  \bibinfo{booktitle}{\emph{{ICML}}}. \bibinfo{publisher}{{PMLR}},
  \bibinfo{pages}{4254--4263}.
\newblock


\bibitem[\protect\citeauthoryear{Rakhsha, Radanovic, Devidze, Zhu, and
  Singla}{Rakhsha et~al\mbox{.}}{2020}]%
        {DBLP:conf/icml/RakhshaRD0S20}
\bibfield{author}{\bibinfo{person}{Amin Rakhsha}, \bibinfo{person}{Goran
  Radanovic}, \bibinfo{person}{Rati Devidze}, \bibinfo{person}{Xiaojin Zhu},
  {and} \bibinfo{person}{Adish Singla}.} \bibinfo{year}{2020}\natexlab{}.
\newblock \showarticletitle{Policy {T}eaching via {E}nvironment {P}oisoning:
  {T}raining-time {A}dversarial {A}ttacks against {R}einforcement {L}earning}.
  In \bibinfo{booktitle}{\emph{ICML}}. \bibinfo{publisher}{{PMLR}},
  \bibinfo{pages}{7974--7984}.
\newblock


\bibitem[\protect\citeauthoryear{Rakhsha, Radanovic, Devidze, Zhu, and
  Singla}{Rakhsha et~al\mbox{.}}{2021}]%
        {DBLP:journals/corr/abs-2011-10824}
\bibfield{author}{\bibinfo{person}{Amin Rakhsha}, \bibinfo{person}{Goran
  Radanovic}, \bibinfo{person}{Rati Devidze}, \bibinfo{person}{Xiaojin Zhu},
  {and} \bibinfo{person}{Adish Singla}.} \bibinfo{year}{2021}\natexlab{}.
\newblock \showarticletitle{Policy {T}eaching in {R}einforcement {L}earning via
  {E}nvironment {P}oisoning {A}ttacks}.
\newblock \bibinfo{journal}{\emph{Journal of Machine Learning Research}}
  \bibinfo{volume}{22}, \bibinfo{number}{210} (\bibinfo{year}{2021}),
  \bibinfo{pages}{1--45}.
\newblock


\bibitem[\protect\citeauthoryear{Randl{\o}v and Alstr{\o}m}{Randl{\o}v and
  Alstr{\o}m}{1998}]%
        {DBLP:conf/icml/RandlovA98}
\bibfield{author}{\bibinfo{person}{Jette Randl{\o}v} {and}
  \bibinfo{person}{Preben Alstr{\o}m}.} \bibinfo{year}{1998}\natexlab{}.
\newblock \showarticletitle{Learning to {D}rive a {B}icycle {U}sing
  {R}einforcement {L}earning and {S}haping}. In
  \bibinfo{booktitle}{\emph{ICML}}. \bibinfo{publisher}{Morgan Kaufmann},
  \bibinfo{pages}{463--471}.
\newblock


\bibitem[\protect\citeauthoryear{Santoro, Bartunov, Botvinick, Wierstra, and
  Lillicrap}{Santoro et~al\mbox{.}}{2016}]%
        {santoro2016meta}
\bibfield{author}{\bibinfo{person}{Adam Santoro}, \bibinfo{person}{Sergey
  Bartunov}, \bibinfo{person}{Matthew Botvinick}, \bibinfo{person}{Daan
  Wierstra}, {and} \bibinfo{person}{Timothy Lillicrap}.}
  \bibinfo{year}{2016}\natexlab{}.
\newblock \showarticletitle{Meta-{L}earning with {M}emory-{A}ugmented {N}eural
  {N}etworks}. In \bibinfo{booktitle}{\emph{ICML}}. PMLR,
  \bibinfo{pages}{1842--1850}.
\newblock


\bibitem[\protect\citeauthoryear{Simsek, Wolfe, and Barto}{Simsek
  et~al\mbox{.}}{2005}]%
        {DBLP:conf/icml/SimsekWB05}
\bibfield{author}{\bibinfo{person}{{\"{O}}zg{\"{u}}r Simsek},
  \bibinfo{person}{Alicia~P. Wolfe}, {and} \bibinfo{person}{Andrew~G. Barto}.}
  \bibinfo{year}{2005}\natexlab{}.
\newblock \showarticletitle{{I}dentifying {U}seful {S}ubgoals in
  {R}einforcement {L}earning by {L}ocal {G}raph {P}artitioning}. In
  \bibinfo{booktitle}{\emph{ICML}}. \bibinfo{publisher}{{ACM}},
  \bibinfo{pages}{816--823}.
\newblock


\bibitem[\protect\citeauthoryear{Sorg, Singh, and Lewis}{Sorg
  et~al\mbox{.}}{2010}]%
        {sorg2010online_gradient_asc}
\bibfield{author}{\bibinfo{person}{Jonathan Sorg}, \bibinfo{person}{Satinder~P.
  Singh}, {and} \bibinfo{person}{Richard~L. Lewis}.}
  \bibinfo{year}{2010}\natexlab{}.
\newblock \showarticletitle{Reward {D}esign via {O}nline {G}radient {A}scent}.
  In \bibinfo{booktitle}{\emph{NeurIPS}}. \bibinfo{pages}{2190--2198}.
\newblock


\bibitem[\protect\citeauthoryear{Sutton and Barto}{Sutton and Barto}{2018}]%
        {sutton2018reinforcement}
\bibfield{author}{\bibinfo{person}{Richard~S. Sutton} {and}
  \bibinfo{person}{Andrew~G. Barto}.} \bibinfo{year}{2018}\natexlab{}.
\newblock \bibinfo{booktitle}{\emph{{R}einforcement {L}earning: {A}n
  {I}ntroduction}}.
\newblock \bibinfo{publisher}{MIT press}.
\newblock


\bibitem[\protect\citeauthoryear{Trott, Zheng, Xiong, and Socher}{Trott
  et~al\mbox{.}}{2019}]%
        {DBLP:conf/nips/TrottZXS19}
\bibfield{author}{\bibinfo{person}{Alexander Trott}, \bibinfo{person}{Stephan
  Zheng}, \bibinfo{person}{Caiming Xiong}, {and} \bibinfo{person}{Richard
  Socher}.} \bibinfo{year}{2019}\natexlab{}.
\newblock \showarticletitle{Keeping {Y}our {D}istance: {S}olving {S}parse
  {R}eward {T}asks {U}sing {S}elf-{B}alancing {S}haped {R}ewards}. In
  \bibinfo{booktitle}{\emph{NeurIPS}}. \bibinfo{pages}{10376--10386}.
\newblock


\bibitem[\protect\citeauthoryear{Wiewiora}{Wiewiora}{2003}]%
        {DBLP:journals/jair/Wiewiora03}
\bibfield{author}{\bibinfo{person}{Eric Wiewiora}.}
  \bibinfo{year}{2003}\natexlab{}.
\newblock \showarticletitle{Potential-{B}ased {S}haping and {Q}-{V}alue
  {I}nitialization are {E}quivalent}.
\newblock \bibinfo{journal}{\emph{Journal of Artificial Intelligence Research}}
   \bibinfo{volume}{19} (\bibinfo{year}{2003}), \bibinfo{pages}{205--208}.
\newblock


\bibitem[\protect\citeauthoryear{Xiao, Lu, Ramasubramanian, Clark, Bushnell,
  and Poovendran}{Xiao et~al\mbox{.}}{2020}]%
        {xiao2020fresh}
\bibfield{author}{\bibinfo{person}{Baicen Xiao}, \bibinfo{person}{Qifan Lu},
  \bibinfo{person}{Bhaskar Ramasubramanian}, \bibinfo{person}{Andrew Clark},
  \bibinfo{person}{Linda Bushnell}, {and} \bibinfo{person}{Radha Poovendran}.}
  \bibinfo{year}{2020}\natexlab{}.
\newblock \showarticletitle{{FRESH}: {I}nteractive {R}eward {S}haping in
  {H}igh-{D}imensional {S}tate {S}paces using {H}uman {F}eedback}. In
  \bibinfo{booktitle}{\emph{AAMAS}}. \bibinfo{publisher}{International
  Foundation for Autonomous Agents and Multiagent Systems},
  \bibinfo{pages}{1512--1520}.
\newblock


\bibitem[\protect\citeauthoryear{Zhang and Parkes}{Zhang and Parkes}{2008}]%
        {DBLP:conf/aaai/ZhangP08}
\bibfield{author}{\bibinfo{person}{Haoqi Zhang} {and} \bibinfo{person}{David~C.
  Parkes}.} \bibinfo{year}{2008}\natexlab{}.
\newblock \showarticletitle{Value-{B}ased {P}olicy {T}eaching with {A}ctive
  {I}ndirect {E}licitation}. In \bibinfo{booktitle}{\emph{AAAI}}.
  \bibinfo{publisher}{{AAAI} Press}, \bibinfo{pages}{208--214}.
\newblock


\bibitem[\protect\citeauthoryear{Zhang, Parkes, and Chen}{Zhang
  et~al\mbox{.}}{2009}]%
        {DBLP:conf/sigecom/ZhangPC09}
\bibfield{author}{\bibinfo{person}{Haoqi Zhang}, \bibinfo{person}{David~C.
  Parkes}, {and} \bibinfo{person}{Yiling Chen}.}
  \bibinfo{year}{2009}\natexlab{}.
\newblock \showarticletitle{Policy {T}eaching through {R}eward {F}unction
  {L}earning}. In \bibinfo{booktitle}{\emph{EC}}. \bibinfo{publisher}{{ACM}},
  \bibinfo{pages}{295--304}.
\newblock


\bibitem[\protect\citeauthoryear{Zhang, Ma, and Singla}{Zhang
  et~al\mbox{.}}{2020}]%
        {DBLP:conf/nips/ZhangMS20}
\bibfield{author}{\bibinfo{person}{Xuezhou Zhang}, \bibinfo{person}{Yuzhe Ma},
  {and} \bibinfo{person}{Adish Singla}.} \bibinfo{year}{2020}\natexlab{}.
\newblock \showarticletitle{Task-{A}gnostic {E}xploration in {R}einforcement
  {L}earning}. In \bibinfo{booktitle}{\emph{NeurIPS}}.
\newblock


\bibitem[\protect\citeauthoryear{Zheng, Oh, and Singh}{Zheng
  et~al\mbox{.}}{2018}]%
        {zheng2018learning}
\bibfield{author}{\bibinfo{person}{Zeyu Zheng}, \bibinfo{person}{Junhyuk Oh},
  {and} \bibinfo{person}{Satinder Singh}.} \bibinfo{year}{2018}\natexlab{}.
\newblock \showarticletitle{On {L}earning {I}ntrinsic {R}ewards for {P}olicy
  {G}radient {M}ethods}. In \bibinfo{booktitle}{\emph{NeurIPS}}.
  \bibinfo{pages}{4649--4659}.
\newblock


\end{thebibliography}


\clearpage    
\onecolumn
\appendix
{
    \allowdisplaybreaks
\section{Appendix}
\label{sec-app:proofs}

In this appendix, we provide additional details and proofs.

\subsection{Additional Details}

\begin{algorithm}[h!]
    \caption{Expert-driven Explicable and Adaptive Reward Design (\AlgOurs{}): Full Implementation}
    \begin{algorithmic}[1]
        \State \textbf{Input:} \looseness-1MDP $M := \big( \mathcal{S},\mathcal{A},T,P_0,\gamma,\overline{R} \big)$, target policy $\pi^T$, RL algorithm $L$, reward constraint set $\mathcal{R}$, first-in-first-out buffer $\mathcal{D}$ with size $D_\text{max}$, reward update rate $N_r$, policy update rate $N_\pi$
        \State \textbf{Initialize:} learner's initial policy $\pi_{0}^L$ 
        \For{$k = 1,2,\dots, K$}
            \Statex \ \ \quad \textcolor{blue}{// reward update}
            \If{$k \% N_r = 0$} 
                \State \looseness-1Expert/teacher updates the reward function by solving the optimization problem in Eq.~\eqref{eq:reward-design-problem}.
            \Else 
                \State Keep previous reward $R_k \gets R_{k-1}$
            \EndIf
            \Statex \ \ \quad \textcolor{blue}{// policy update}
            \If{$k \% N_\pi = 0$} 
                \State Learner updates the policy: $\pi^L_k \gets L(\pi^L_{k-1}, R_k)$ using the latest rollouts in $\mathcal{D}$
            \Else
                \State Keep previous policy $\pi_{k}^L \gets \pi_{k-1}^L$
            \EndIf
            \Statex \ \ \quad \textcolor{blue}{// data collection}
            \State \parbox[t]{\dimexpr\linewidth-\algorithmicindent}{Rollout the policy $\pi_{k}^L$ in the MDP $M$ to obtain a trajectory $\xi^k = \brr{s_0^k, a_0^k, s_1^k, a_1^k, \dots, s_H^k}$}
            \State \parbox[t]{\dimexpr\linewidth-\algorithmicindent}{Add $\xi^k$ to the buffer $\mathcal{D}$ (the oldest trajectory gets removed when the buffer $\mathcal{D}$ is full) }
        \EndFor{}
        \State \textbf{Output:} learner's policy $\pi^L_K$
    \end{algorithmic}
    \label{alg:expert-driven-adaptive-reward-design-full}
\end{algorithm}

\paragraph{\textbf{Implementation details.}} In Algorithm~\ref{alg:expert-driven-adaptive-reward-design-full}, we present an extended version of Algorithm~\ref{alg:expert-driven-adaptive-reward-design} with full implementation details. In our experiments in the main paper (Section~\ref{sec:evaluation}), we used $N_\pi=2$ and $N_r=5$ similar to hyperparameters considered in existing works on self-supervised reward design~\cite{zheng2018learning,memarian2021self,devidze2022exploration}. Here, we update the policy more frequently than the reward for stability of the learning process. We also conducted additional experiments with values of $N_r=100$ and $N_r=1000$, while keeping $N_\pi=2$. These increased values of $N_r$ lowered the variance without affecting the overall performance. In general, there is limited theoretical understanding of the impact of adaptive rewards on learning process stability, and it would be interesting to investigate this in future work.

\paragraph{\textbf{Impact of horizon $h$ in Eq.~\eqref{eq:reward-design-problem}.}} For the lower $h$ values, the reward design process can provide stronger reward signals, thereby speeding up the learning process. The choice $h=1$ further simplifies the computation of $I_h$ as it doesn't require computing the $h$-step advantage value function. In general, when we are designing rewards for different types of learners, it could be more effective to use the informativeness criterion summed up over different values of $h$.

\paragraph{\textbf{Structural constraints in $\mathcal{R}$.}} Given a feature representation $f: \mathcal{S} \times \mathcal{A} \to \{0,1\}^d$, we employed parametric reward functions of the form $R_\phi(s,a) = \ipp{\phi}{f(s,a)}$ in our experiments in the main paper (Section~\ref{sec:evaluation}). A general methodology to create feature representations could be based on state/action abstractions. It would be interesting to automatically learn such abstractions and quantify the size of the required abstraction for a given environment. \AlgOurs{} framework can be extended to incorporate structural constraints such as those defined by a set of logical rules defined over state/action space. When the set of logical rules induces a partition over the state-action space, we can define a one-hot feature representation over this partitioned space.

\paragraph{\textbf{Computational complexity of different techniques.}} At every step $k$ of Algorithm~\ref{alg:expert-driven-adaptive-reward-design}, \AlgOurs{} requires to solve the optimization problem in Eq.~\eqref{eq:reward-design-problem}. However, since it is a linearly constrained concave-maximization problem w.r.t. $R \in \mathbb{R}^{|S|\times|A|}$, it can be efficiently solved using standard convex programs (similar to the inner problem $(P1)$ in \EXPRD{}~\cite{devidze2021explicable}). Notably, the rewards ${R}^{\orig}$, ${R}^{\Invariance}$, and ${R}^{\EXPRD}$ are agnostic to the learner's policy and remain constant throughout the training process. ${R}^{\Invariance}$ baseline technique is based on related work~\cite{DBLP:conf/icml/RakhshaRD0S20,DBLP:journals/corr/abs-2011-10824} when not considering informativeness. This baseline still requires solving a constrained optimization problem and has a similar computational complexity of ${R}^{\EXPRD}$.

\subsection{Proof of Proposition~\ref{prop:intuitive-grad}}

\begin{proof}
For the simple learning algorithm $L$, we can write the derivative of the informativeness criterion in 
Eq.~\eqref{eq:intuitive-IR-bi-level} as follows:
\begin{align}
\bss{\nabla_\phi I_L(R_\phi \mid \overline{R}, \pi^T, \pi^L)}_{\phi} 
~\stackrel{(a)}{=}~& \bss{\nabla_\phi \theta^L_\textnormal{new} (\phi) \cdot \nabla_{\theta^L_\textnormal{new} (\phi)} J(\pi_{\theta^L_\textnormal{new}(\phi)}; \overline{R}, \pi^T)}_{\phi} \nonumber \\ 
~\stackrel{(b)}{\approx}~& \underbrace{\bss{\nabla_\phi \theta^L_\textnormal{new} (\phi)}_{\phi}}_{\textcircled{1}} \cdot \underbrace{\bss{\nabla_{\theta} J(\pi_{\theta}; \overline{R}, \pi^T)}_{\theta^L}}_{\textcircled{2}} , \nonumber
\end{align}
where the equality in $(a)$ is due to chain rule, and the approximation in $(b)$ assumes a smoothness condition of $\Big\lVert \bss{\nabla_{\theta} J(\pi_{\theta}; \overline{R}, \pi^T)}_{\theta^L_\textnormal{new}(\phi)} - \bss{\nabla_{\theta} J(\pi_{\theta}; \overline{R}, \pi^T)}_{\theta^L} \Big\rVert_2 \leq c \cdot \norm{\theta^L_\textnormal{new}(\phi) - \theta^L}_2$ for some $c > 0$. For the $L$ described above, we can obtain intuitive forms of the terms \textcircled{1} and \textcircled{2}. For any $s \in \mathcal{S}, a \in \mathcal{A}$, let $\mathbf{1}_{s,a} \in \mathbb{R}^{\abs{\mathcal{S}} \cdot \abs{\mathcal{A}}}$ denote a vector with $1$ in the $(s,a)$-th entry and $0$ elsewhere. 

First, we simplify the term \textcircled{1} as follows:
\begin{align*}
\bss{\nabla_\phi \theta^L_\textnormal{new} (\phi)}_{\phi}  \stackrel{(a)}{=}& \alpha \cdot \Expectover{\mu_{s,a}^{\pi^L}}{\big[ \nabla_\phi {Q}^{\pi^L}_{R_\phi,h} (s,a) \big]_{\phi} \cdot \big[\nabla_\theta \log \pi_\theta (a|s)\big]_{\theta^L}^\top} \\
=~& \alpha \cdot \Expectover{\mu_{s}^{\pi^L}}{\sum_{a}{\pi^L(a|s) \cdot \big[ \nabla_\phi {Q}^{\pi^L}_{R_\phi,h} (s,a) \big]_{\phi} \cdot \big[\nabla_\theta \log \pi_\theta (a|s)\big]_{\theta^L}^\top}} \\
\stackrel{(b)}{=}& \alpha \cdot \Expectover{\mu_{s}^{\pi^L}}{\sum_{a}{\pi^L(a|s) \cdot \big[ \nabla_\phi {Q}^{\pi^L}_{R_\phi,h} (s,a) \big]_{\phi} \cdot \mathbf{1}_{s,a}^\top}} - \Expectover{\mu_{s}^{\pi^L}}{\sum_{a}{\pi^L(a|s) \cdot \big[ \nabla_\phi {Q}^{\pi^L}_{R_\phi,h} (s,a) \big]_{\phi} \cdot \Big( \sum_{a'}{\pi^L (a'|s) \cdot \mathbf{1}_{s,a'}^\top} \Big)}} \\
=~& \alpha \cdot \Expectover{\mu_{s}^{\pi^L}}{\sum_{a}{\pi^L(a|s) \cdot \big[ \nabla_\phi {Q}^{\pi^L}_{R_\phi,h} (s,a) \big]_{\phi} \cdot \mathbf{1}_{s,a}^\top}} - \Expectover{\mu_{s}^{\pi^L}}{\Big[ \nabla_\phi \sum_{a}{\pi^L(a|s) \cdot {Q}^{\pi^L}_{R_\phi,h} (s,a)} \Big]_{\phi} \cdot \Big( \sum_{a'}{\pi^L (a'|s) \cdot \mathbf{1}_{s,a'}^\top} \Big)} \\
=~& \alpha \cdot \Expectover{\mu_{s}^{\pi^L}}{\sum_{a}{\pi^L(a|s) \cdot \big[ \nabla_\phi {Q}^{\pi^L}_{R_\phi,h} (s,a) \big]_{\phi} \cdot \mathbf{1}_{s,a}^\top}} - \Expectover{\mu_{s}^{\pi^L}}{\big[ \nabla_\phi {V}^{\pi^L}_{R_\phi,h} (s) \big]_{\phi} \cdot \Big( \sum_{a'}{\pi^L (a'|s) \cdot \mathbf{1}_{s,a'}^\top} \Big)} \\
\stackrel{(c)}{=}& \alpha \cdot \Expectover{\mu_{s}^{\pi^L}}{\sum_{a}{\pi^L(a|s) \cdot \big[ \nabla_\phi {Q}^{\pi^L}_{R_\phi,h} (s,a) \big]_{\phi} \cdot \mathbf{1}_{s,a}^\top}} - \Expectover{\mu_{s}^{\pi^L}}{\sum_{a}{\pi^L (a|s) \cdot \big[ \nabla_\phi {V}^{\pi^L}_{R_\phi,h} (s) \big]_{\phi} \cdot \mathbf{1}_{s,a}^\top}} \\
~=~& \alpha \cdot \Expectover{\mu_{s,a}^{\pi^L}}{\big[ \nabla_\phi {Q}^{\pi^L}_{R_\phi,h} (s,a) \big]_{\phi} \cdot \mathbf{1}_{s,a}^\top - \big[ \nabla_\phi {V}^{\pi^L}_{R_\phi,h} (s) \big]_{\phi} \cdot \mathbf{1}_{s,a}^\top} \\
~=~& \alpha \cdot \Expectover{\mu_{s,a}^{\pi^L}}{\bss{\nabla_\phi {{A}^{\pi^L}_{R_\phi,h} (s,a)}}_{\phi} \cdot \mathbf{1}_{s,a}^\top} .
\end{align*}
The equality in $(a)$ arises from the meta-gradient derivations presented in~\cite{andrychowicz2016learning,santoro2016meta,nichol2018first}. In $(b)$, the equality is a consequence of the relationship $\big[\nabla_\theta \log \pi_\theta (a|s)\big]_{\theta^L} = \Big( \mathbf{1}_{s,a} - \sum_{a'}{\pi^L (a'|s) \cdot \mathbf{1}_{s,a'}} \Big)$. Finally, in $(c)$, the equality can be attributed to the change of variable from $a'$ to $a$. Then, by applying analogous reasoning to the above discussion, we simplify the term \textcircled{2} as follows:
\begin{align*}
\bss{\nabla_{\theta} J(\pi_{\theta}; \overline{R}, \pi^T)}_{\theta^L} 
~=~& \Expectover{\mu_{s}^{\pi^T}}{\sum_a {A}^{\pi^T}_{\overline{R}}(s,a) \cdot \big[\nabla_\theta \pi_\theta (a|s)\big]_{\theta^L}} \\
~=~& \Expectover{\mu_{s}^{\pi^T}}{\sum_{a}\pi^L(a|s) \cdot {A}^{\pi^T}_{\overline{R}}(s,a) \cdot \big[\nabla_\theta \log \pi_\theta (a|s)\big]_{\theta^L}} \\
~=~& \Expectover{\mu_{s}^{\pi^T}}{\sum_{a}{\pi^L(a|s) \cdot {A}^{\pi^T}_{\overline{R}}(s,a) \cdot \mathbf{1}_{s,a}}} - \Expectover{\mu_{s}^{\pi^T}}{\sum_{a}{\pi^L(a|s) \cdot {A}^{\pi^T}_{\overline{R}}(s,a) \cdot \Big( \sum_{a'}{\pi^L (a'|s) \cdot \mathbf{1}_{s,a'}} \Big)}} \\
~=~& \Expectover{\mu_{s}^{\pi^T}}{\sum_{a}{\pi^L(a|s) \cdot {A}^{\pi^T}_{\overline{R}}(s,a) \cdot \mathbf{1}_{s,a}}} - \Expectover{\mu_{s}^{\pi^T}}{{A}^{\pi^T}_{\overline{R}}(s, \pi^L(s)) \cdot \Big( \sum_{a'}{\pi^L (a'|s) \cdot \mathbf{1}_{s,a'}} \Big)} \\
~=~& \Expectover{\mu_{s}^{\pi^T}}{\sum_{a}{\pi^L(a|s) \cdot {A}^{\pi^T}_{\overline{R}}(s,a) \cdot \mathbf{1}_{s,a}}} - \Expectover{\mu_{s}^{\pi^T}}{\sum_{a}{\pi^L (a|s) \cdot {A}^{\pi^T}_{\overline{R}}(s, \pi^L(s)) \cdot \mathbf{1}_{s,a}}} \\
~=~& \Expectover{\mu_{s}^{\pi^T}}{\sum_{a}{\pi^L(a|s) \cdot \brr{{A}^{\pi^T}_{\overline{R}}(s,a) - {A}^{\pi^T}_{\overline{R}}(s, \pi^L(s))} \cdot \mathbf{1}_{s,a}}} .
\end{align*}
Finally, by taking the matrix product of the terms \textcircled{1} and \textcircled{2}, we have the following:
\begin{align*}
\bss{\nabla_\phi \theta^L_\textnormal{new} (\phi)}_{\phi} \cdot \bss{\nabla_{\theta} J(\pi_{\theta}; \overline{R}, \pi^T)}_{\theta^L} & \\
~=~& \alpha \cdot \Big( \sum_{s',a'}{\mu_{s',a'}^{\pi^L} \cdot \bss{\nabla_\phi {A}^{\pi^L}_{R_\phi,h} (s',a')}_{\phi} \cdot \mathbf{1}_{s',a'}^\top} \Big) \cdot \Big( \sum_{s,a}{\mu_{s}^{\pi^T} \cdot \pi^L(a|s) \cdot \brr{{A}^{\pi^T}_{\overline{R}}(s,a) - {A}^{\pi^T}_{\overline{R}}(s, \pi^L(s))} \cdot \mathbf{1}_{s,a}} \Big) \\
~=~& \alpha \cdot \sum_{s,a} \mu_{s}^{\pi^T} \cdot \pi^L(a|s) \cdot \brr{{A}^{\pi^T}_{\overline{R}}(s,a) - {A}^{\pi^T}_{\overline{R}}(s, \pi^L(s))} \cdot \Big( \sum_{s',a'}{\mu_{s',a'}^{\pi^L} \cdot \bss{\nabla_\phi {A}^{\pi^L}_{R_\phi,h} (s',a')}_{\phi} \cdot \mathbf{1}_{s',a'}^\top} \Big) \cdot \mathbf{1}_{s,a} \\
~=~& \alpha \cdot \sum_{s,a}{\mu_{s}^{\pi^T} \cdot \pi^L(a|s) \cdot \brr{{A}^{\pi^T}_{\overline{R}}(s,a) - {A}^{\pi^T}_{\overline{R}}(s, \pi^L(s))} \cdot \mu_{s,a}^{\pi^L} \cdot \bss{\nabla_\phi {A}^{\pi^L}_{R_\phi,h} (s,a)}_{\phi}} \\
~=~& \alpha \cdot \Expectover{\mu_{s,a}^{\pi^L}}{\mu_{s}^{\pi^T} \cdot \pi^L(a|s) \cdot \brr{{A}^{\pi^T}_{\overline{R}}(s,a) - {A}^{\pi^T}_{\overline{R}}(s, \pi^L(s))} \cdot \bss{\nabla_\phi {A}^{\pi^L}_{R_\phi,h} (s,a)}_{\phi}} ,
\end{align*}
which completes the proof.
\end{proof}

\subsection{Proof of Theorem~\ref{thm:shaping-complexity}}

\begin{proof}
For any fixed policy $\pi^L$, consider the following reward design problem:
\[
\max_{R \in \mathcal{R}} I_{h=1}(R \mid \overline{R}, \pi^T, \pi^L) ,
\]
where $\mathcal{R} = \bcc{R: \abs{R\brr{s,a}} \leq R_\mathrm{max}, \forall s \in \mathcal{S}, a \in \mathcal{A}}$. Since $I_{h=1}(R \mid \overline{R}, \pi^T, \pi^L) = \sum_s \mu^{\pi^T}_s \cdot \mu^{\pi^L}_s \cdot \sum_a \bcc{\pi^L(a|s)}^2 \cdot \brr{A^{\pi^T}_{\overline{R}}(s,a) - \Expectover{\pi^L(b|s)}{A^{\pi^T}_{\overline{R}}(s,b)}} \cdot \brr{R (s,a) - \Expectover{\pi^L(b|s)}{R (s,b)}}$, reward values for each state $s$ can be independently optimized. Thus, for each state $s \in \mathcal{S}$, we solve the following problem independently to find optimal values for $R(s, a), \forall a$:
\[
\max_{R \in \mathcal{R}} \sum_a \bcc{\pi^L(a|s)}^2 \cdot \brr{A^{\pi^T}_{\overline{R}}(s,a) - \Expectover{\pi^L(b|s)}{A^{\pi^T}_{\overline{R}}(s,b)}} \cdot \brr{R (s,a) - \Expectover{\pi^L(b|s)}{R (s,b)}} . 
\]
The above problem can be further simplified by gathering the terms involving $R(s,a)$: 
\[
\max_{R(s,a): \abs{R(s,a)} \leq R_\mathrm{max}} \pi^L(a|s) \cdot Z (s,a) \cdot R (s,a) ,
\]
where 
\begin{align*}
Z (s,a) ~=~& \pi^L(a|s) \cdot \brr{A^{\pi^T}_{\overline{R}}(s,a) - \Expectover{\pi^L(b|s)}{A^{\pi^T}_{\overline{R}}(s,b)}} - \sum_{a'} \bcc{\pi^L(a'|s)}^2 \cdot \brr{A^{\pi^T}_{\overline{R}}(s,a') - \Expectover{\pi^L(b|s)}{A^{\pi^T}_{\overline{R}}(s,b)}} .
\end{align*}
Then, we have the following solution for the reward design problem: 
\[
R(s,a) ~=~ \begin{cases}
+ R_\mathrm{max}, & \text{if } Z(s,a) \geq 0 \\
- R_\mathrm{max}, & \text{otherwise} .
\end{cases}
\]
We conduct the following "worst-case" analysis for Algorithm~\ref{alg:expert-driven-adaptive-reward-design}. For any state $s \in \mathcal{S}$:
\begin{enumerate}
\item reward update at step $k=1$: for the randomly initialized policy $\pi^L_0(a|s) = 1/\abs{\mathcal{A}}, \forall a \in \mathcal{A}$, for the highly sub-optimal action $a_1$ w.r.t. $\pi^T$, we have $Z(s, a_1) = \frac{1}{\abs{\mathcal{A}}} \cdot \brr{A^{\pi^T}_{\overline{R}}(s,a_1) - \Expectover{\pi^L_0(b|s)}{A^{\pi^T}_{\overline{R}}(s,b)}} - \sum_{a'} \frac{1}{\abs{\mathcal{A}}^2} \cdot \brr{A^{\pi^T}_{\overline{R}}(s,a') - \Expectover{\pi^L_0(b|s)}{A^{\pi^T}_{\overline{R}}(s,b)}} = \frac{1}{\abs{\mathcal{A}}} \cdot \brr{A^{\pi^T}_{\overline{R}}(s,a_1) - \Expectover{\pi^L_0(b|s)}{A^{\pi^T}_{\overline{R}}(s,b)}} < 0$. Then, the updated reward will be $R_1(s, a_1) = - R_\mathrm{max}$ and $R_1(s, a) = + R_\mathrm{max}, \forall a \in \mathcal{A} \backslash \bcc{a_1}$. 
\item policy update at step $k=1$: for the updated reward function $R_1$, the policy $\pi^L_1 (s) \gets \argmax_a R_1 (s,a)$ is given by $\pi^L_1(a_1|s) = 0$ and $\pi^L_1(a|s) = 1/(\abs{\mathcal{A}} - 1), \forall a \in \mathcal{A} \backslash \bcc{a_1}$. 
\item reward update at step $k=2$: for the updated policy $\pi^L_1$, for the second highly sub-optimal action $a_2$ w.r.t. $\pi^T$ also, we have $Z(s, a_2) = \frac{1}{(\abs{\mathcal{A}} - 1)} \cdot \brr{A^{\pi^T}_{\overline{R}}(s,a_2) - \Expectover{\pi^L_1(b|s)}{A^{\pi^T}_{\overline{R}}(s,b)}} - \sum_{a' \in \mathcal{A} \backslash \bcc{a_1}} \frac{1}{(\abs{\mathcal{A}} - 1)^2} \cdot \brr{A^{\pi^T}_{\overline{R}}(s,a') - \Expectover{\pi^L_1(b|s)}{A^{\pi^T}_{\overline{R}}(s,b)}} = \frac{1}{(\abs{\mathcal{A}} - 1)} \cdot \brr{A^{\pi^T}_{\overline{R}}(s,a_2) - \Expectover{\pi^L_1(b|s)}{A^{\pi^T}_{\overline{R}}(s,b)}} < 0$. Then, the updated reward will be $R_2(s, a_1) = R_2(s, a_2) = - R_\mathrm{max}$ and $R_2(s, a) = + R_\mathrm{max}, \forall a \in \mathcal{A} \backslash \bcc{a_1, a_2}$.
\item policy update at step $k=2$: for the updated reward function $R_2$, the policy $\pi^L_2 (s) \gets \argmax_a R_2 (s,a)$ is given by $\pi^L_2(a_1|s) = \pi^L_2(a_2|s) = 0$ and $\pi^L_2(a|s) = 1/(\abs{\mathcal{A}} - 2), \forall a \in \mathcal{A} \backslash \bcc{a_1, a_2}$.
\end{enumerate}
By continuing the above argument for $\abs{A}$ steps, we can show that $\pi^L_k$ converges to the target policy $\pi^T$, which completes the proof.   
\end{proof}

}

\end{document}